\documentclass[10pt]{article} % For LaTeX2e
\usepackage[preprint]{tmlr}

\usepackage{amsmath,amsfonts,bm}

\def\eqref#1{equation~\ref{#1}}
\def\1{\bm{1}}

\DeclareMathAlphabet{\mathsfit}{\encodingdefault}{\sfdefault}{m}{sl}
\SetMathAlphabet{\mathsfit}{bold}{\encodingdefault}{\sfdefault}{bx}{n}

\usepackage{mathtools}
\usepackage{graphicx}
\usepackage{eurosym}
\usepackage{tabularray}
\usepackage{algorithm}
\usepackage{dsfont}
\usepackage{multicol}
\usepackage{multirow} 
\let\AND\relax
\usepackage{algorithmic}
\usepackage{microtype}
\usepackage{subfigure}
\usepackage{subcaption}
\usepackage{booktabs} % for professional tables
\usepackage{hyperref}
\usepackage{amsmath}
\usepackage{amssymb}
\usepackage{mathtools}
\usepackage{amsthm}
\usepackage[capitalize,noabbrev]{cleveref}
\usepackage{caption}
\usepackage{comment}
\theoremstyle{plain}
\newtheorem{theorem}{Theorem}[section]

\newtheorem{lemma}[theorem]{Lemma}
\newtheorem{corollary}[theorem]{Corollary}
\theoremstyle{definition}
\newtheorem{definition}[theorem]{Definition}

\theoremstyle{remark}

\usepackage{wrapfig}
\usepackage{floatflt}  % in preamble
\usepackage[normalem]{ulem}

\title{Provably Efficient Reward Transfer in Reinforcement Learning with Discrete Markov Decision Processes}

\author{\name Kevin Vora \email kvora1@asu.edu \\
      \addr School of Computing and Augmented Intelligence \\
      Arizona State University
      \AND \\
      \name Yu Zhang \email yu.Zhang.442@asu.edu \\
      \addr School of computing and Augmented Intelligence \\
      Arizona State University}

\def\month{MM}  % Insert correct month for camera-ready version
\def\year{YYYY} % Insert correct year for camera-ready version
\def\openreview{\url{https://openreview.net/forum?id=XXXX}} % Insert correct link to OpenReview for camera-ready version

\begin{document}

\maketitle

\begin{abstract}

In this paper, we propose a new solution to reward adaptation (RA) in reinforcement learning, where the agent adapts to a target reward function based on one or more existing source behaviors learned a priori under the same domain dynamics but different reward functions. While learning the target behavior from scratch is possible, it is often inefficient given the available source behaviors. Our work introduces a new approach to RA through the manipulation of Q-functions. Assuming the target reward function is a known function of the source reward functions, we compute bounds on the Q-function and present an iterative process (akin to value iteration) to tighten these bounds. 
% \sout{This process assumes access to a lite-model, which is easy to provide or learn. }
% during the training of source behaviors. 
% While model-free methods can be sample inefficient and model-based methods are prone to model error, our approach bridges the gap between the two by using the lite model for RA. 
% Importantly, this lite-model is not sufficient for full planning but allows efficient computation of bounds. 
Such bounds enable action pruning in the target domain before learning even starts.
We refer to this method as ``\textit{Q-Manipulation}'' (Q-M). 
\textcolor{black}{The iteration process assumes access to a lite-model, which is easy to provide or learn. }
We formally prove that Q-M, under discrete domains, does not affect the optimality of the returned policy and show that it is provably efficient in terms of sample complexity in a probabilistic sense. Q-M is evaluated in a variety of synthetic and simulation domains to demonstrate its effectiveness, generalizability, and practicality. 
\end{abstract}

\section{Introduction}
Reinforcement Learning (RL) as described by ~\cite{watkins1989learning, sutton2018reinforcement} represents a class of learning methods that allow agents to learn from interacting with the environment. RL has demonstrated great successes in various domains such as games like Chess in \cite{campbell2002deep}, Go in \cite{silver2016mastering}, and Atari games in \cite{mnih2015human}, logistics in ~\cite{yan2022reinforcement}, biology in ~\cite{angermueller2019model}, and robotics in ~\cite{kober2013reinforcement}. 
{\color{black}
However, applying RL to many real-world problems still
suffers from the issue of high sample complexity. 
Prior approaches have been proposed to alleviate the issue from different perspectives, 
such as learning optimization, transfer learning, 
modular and hierarchical RL, and offline RL. However, few methods provably improve sample complexity.

}
\begin{floatingfigure}[r]{0.39\textwidth}
  \centering
  \vskip -15pt
  \includegraphics[width=0.38\textwidth]{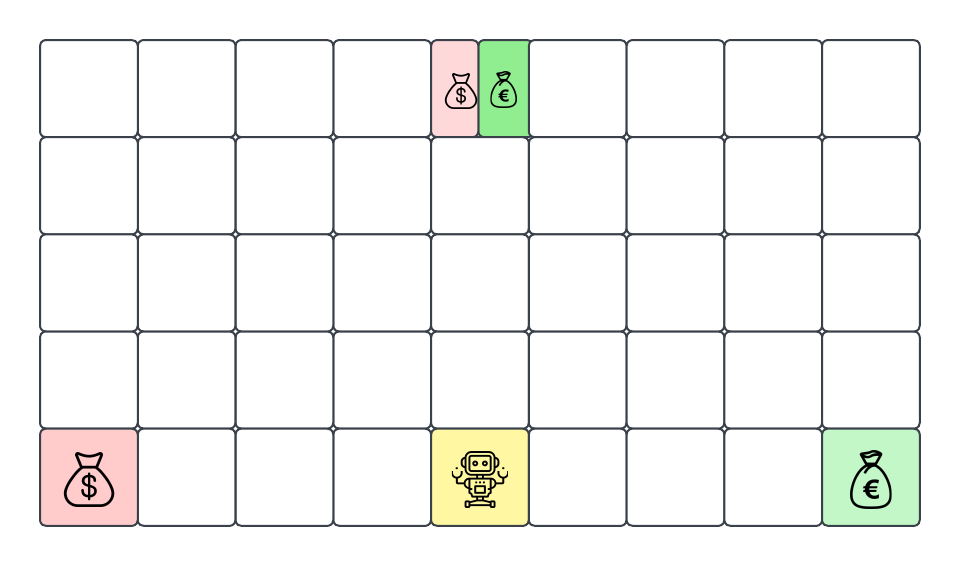}
  \vskip -5pt
  \caption{Dollar-Euro domain.}
  % \vskip -3pt
  \label{env}
\end{floatingfigure}

% However, applying RL to real-world problems is still challenging since the agents are often required to interact with the physical world to learn,
% which can be 
% %expensive. 
% addressed by learning optimization, transfer learning (including domain adaptation), model reuse, offline RL, etc.
% The key to addressing such a problem is to reduce (online) sample complexity.
%, which can often be achieved in one of the following ways in RL:1) learning optimization, 2) transfer learning (including domain adaptation), 3) model reuse, and, more recently, 4) offline RL.
% \vskip-10pt
The problem of reward adaptation (RA) was first introduced and addressed by~\cite{barreto2018successor, barreto2020fast},
where the learning agent adapts to a target reward function given one or multiple existing behaviors learned a priori (referred to as the source behaviors) {\color{black} under the same actions and transition dynamics} but different reward functions. 
% RA has many useful applications, such as adapting an autonomous driving agent to drive fast and safe when it already knows how to operate either fast or safe.
 \textcolor{black}{ RA has many useful applications, such as enabling a vehicle's driving behavior from two known behaviors (comfortable driving with passengers and fast driving for goods delivery) to a new target behavior that combines comfort and speed, accommodating both passengers and goods.}
Featuring such a special type of transfer learning, RA methods can benefit from an ever-growing repertoire of source behaviors to create new and potentially more complex target behaviors.
Learning the target behavior from scratch via model-free methods is possible but often inefficient
given the available source behaviors,
while model-based methods are prone to model error, e.g., due to uneven exploration of the state space when learning the source behaviors.
%\textcolor{red}{Leveraging prior experience, humans are capable of implicitly disregarding actions that are unlikely to contribute to an optimal solution. Similarly, an intelligent agent should be designed to autonomously filter out suboptimal actions based on contextual relevance and learned priors.}
In this paper, 
{\color{black}
 we present a new approach that takes the advantage of both sides while avoiding their limitations to bridge the gap for RA. 
 %\sout{offering a unique solution that} \sout{offers its }
 This lite-model offers
 %\sout{unique and} 
complimentary benefits %\sout{compared} 
with respect} to the previous work.  %\sout{on RA}. 

To better conceptualize the RA problem, consider a grid-world as shown in Fig. \ref{env}, which is an expansion of the Dollar-Euro domain described by \cite{russell2003q}. In this domain, the agent can 
move to any of its adjacent locations at any step. 
%take 4 actions corresponding to the rook’s move in chess. 
The agent’s initial location is colored in yellow, and the terminal locations are colored pink or green,
which correspond to the source reward functions (i.e., collecting dollars and euros), respectively. 
Visiting the terminal location with a single color returns a reward of $1.0$ under the corresponding reward function, 
and visiting the terminal location with split colors returns a reward of $0.6$ under both reward functions. 
%, respectively under each function.
% In RA, the assumption is that the optimal behaviors under the source reward functions are given, referred to as the source behaviors.  
A target domain %, in this example, 
may correspond to a reward function that awards both dollars and euros. 

%Although the problem setting resembles that of Q-ecomposition~\cite{russell2003q},
%Such a problem may resemble that in D-decomposition. However, 
%the focus of Q-Decomposition is to learn under each source reward function independently while constructing the target behavior. 
%In contrast, RA is focused on how to leverage the source behaviors that are available while learning the target behavior. 

%\color{black}\sout{Within the transfer learning for RL literature ~\cite{taylor2009transfer, wulfmeier2023foundations}, our approach lies within the realm of value function transfer approaches for reward adaptation. Prior work leverages value functions from source tasks  in various ways to enable transfer, such as restricting updates \cite{adamczyk2024boosting}, constructing bootstrapped models \cite{barreto2018successor}, composing values \cite{haarnoja2018composable}, or accelerating learning through fine-tuning \cite{narvekar2020curriculum}. We focus on some of them that are most relevant below.}
\textcolor{black}{RA is most efficient when the source and target domains feature reward functions that are correlated. For example, an assumption was made in Successor Feature Q-Learning (SFQL)~\cite{barreto2018successor, barreto2020fast} where}
%\sout{Under the assumption that} 
the reward functions are expressed as %\sout{in the form of}
feature weights. 
\textcolor{black}{An advantage of this assumption is} that the source behaviors can be evaluated easily under the target domain. 
% \sout{a well known prior work for addressing RA} 
SFQL can be viewed as combining the best parts of the source behaviors to initialize learning. 
% \sout{referred to as Successor Feature Q-Learning (SFQL) by ~\cite{barreto2018successor, barreto2020fast}.}
Consequently, SFQL may not work well for situations where the target behavior differs substantially from the source behaviors, such as in the Dollar-Euro domain. 
Our proposed approach, on the other hand, represents a more general knowledge transfer method
whose efficacy does not rely on the similarity between the source and target behaviors. 
%\sout{leading it to be provably beneficial to sample complexity in a probabilistic sense~\cite{kakade2003sample}.}
Soft Q Bounding (SQB)~\cite{adamczyk2024boosting} 
%\sout{bears similarities with ours as it}
establishes double sided bounds of the Q function to speed up learning by clipping overly optimistic or pessimistic updates.   
% \cite{adamczyk2024boosting} is another approach which works on establishing double sided bounds to speed up learning by clipping overly optimisitic or pessimistic updates to Q. SQB computes bounds online in the target domain given initial Q (may or may not be from source behaviors). 
% It is not very clear when clipping Q-values can speed up learning or which source behaviors are advantageous for this approach.
\textcolor{black}{When viewed as a method for RA, it relies on a given Q function computed from source behaviors. Even though the idea of computing the bounds bears some similarity to our work, SQB does not exploit the source behaviors during target learning, resulting in unpredictable performance} \textcolor{black}{ due to reliance on samples.}
%\sout{because Q-value clipping is inherently sample-dependent in model-free settings, the extent to which SQB improves performance remains uncertain.In contrast to SFQL and SQB, our approach considers best- and worst-case outcomes under each source domain and synthesizes this information to compute upper and lower bounds on the target Q-function, enabling action pruning prior to any learning in the target domain.}

% Our approach, instead, 
% reasons about the best/worse-case scenarios
% under each source domain and combines such knowledge to compute upper/lower bounds of the target Q-function to enable action pruning \textcolor{blue}{(before learning in target domain begins)}.
% It results in more 
% }

Our approach to RA is referred to as ``{\textit{Q-Manipulation}}'' (Q-M). 
\textcolor{black}{In this paper, we focus on discrete Markov Decision Processes (MDPs) for a theoretical treatment. 
Challenges in addressing continuous MDPs are discussed in Section \ref{Discussion}, with directions outlined for future work.
% Addressing continuous MDPs will be discussed in Sec. \ref{Discussion} and studied in detail in future work.
}
% \textcolor{red}{In the most general case, Q-M assumes access to one step reachable neighbours (agent's actions result in transitions confined to the local neighborhood and other parts of state space are unreachable) to compute upper and lower bounds on Q-function in target domain. We extend this theory by assuming that source and target rewards are related.
% }
\textcolor{black}{Similar to prior work on RA, we assume the relationship between the source reward functions and the target reward function is known, and in our case, via what is referred to as a {\it combination function}.}
%\sout{We assume the existence of a function\textcolor{black}{, potentially noisy}, referred to as the {\it combination function}, that relates the source reward functions to the target reward function.}
%Note that this differs from the assumption made in~\cite{barreto2018successor},
%which assumes knowledge of the target reward function. 
%Such an assumption is based on the intuition that, 
\textcolor{black}{The intuition here is that }
we often have a good idea about the functional relationship \textcolor{black}{(potentially noisy)} between the source and target reward functions, e.g., linear in the Dollar-Euro domain. 
% target reward function (e.g., consider autonomous driving) but have a good idea about its functional relationship to the source rewards (e.g., safety takes precedence over other criteria, such as driving fast). 
Based on such a relationship, Q-M computes an upper and lower bound of the Q-function in the target domain via an iterative process similar to value iteration. %\sout{to identify actions that cannot contribute to the optimal behavior.} 
%It is important to note that target reward function only informs about the combination function of source rewards but not its values. 
\textcolor{black}{This process operates on a lite-model of the transition function that is easier to provide or estimate while learning the source behaviors. The output bounds}
%It computes upper and lower bounds through an iterative process (similar to value iteration which we define in section \ref{QM-iteratio}). 
%Q-M assumes access to reachable state information for this computation.
%The initial value for this approach is set such that $Q^{UB}>Q^*$ (where $Q^*$ is optimal Q-value) and $Q^{\mu*}<Q^{LB}<Q^*$ (where $Q^{\mu*}$ is optimal value for minimizing discounted sum of rewards). 
%Setting the initial value for this approach plays an important role and we prescribe different methods of initialization in our methodology.
%In case this initialization is not feasible given limited domain knowledge, we also propose initialization techniques with certain assumptions on target reward function and information available from the library of behaviors.
%We assume that the learning agent maintains variants of the Q function (referred to as $Q^{*}$, $Q^{\mu*}$, and $Q_{|R|}^{*}$) for each source behavior: they are computed a priori when the source behavior is learned. We note that all RL methods using value function estimates in learning have access to Q and, with minor modifications, to $Q^{\mu*}$ and $Q_{|R|}^{*}$ as well (more details later).We further assume that the target reward function is a monotonic function of the source reward functions.
%The  $Q^{*}$, $Q^{\mu*}$, and $Q_{|R|}^{*}$ under the source reward functions are used to compute initial $Q^{UB}$ and $Q^{LB}$. % given the source Q functions. 
%We then tighten the bounds using our iteration updates to maximize pruning opportunities. 
enable us to prune actions before learning the target behavior, without affecting its optimality.
%under discrete domains.
%and linear combination functions.
%leading to guaranteed improvements in sample complexity. 
 %We observe that the efficiency of pruning depends on the Average Stochastic Branching Factor of the domain. 
 In our evaluation, we empirically show that the effectiveness of Q-M across simulated and randomly generated domains, and also analyze its limitations. %\sout{focusing on conditions under which its efficacy is negatively impacted. }
% Furthermore, we demonstrate that Q-M can still be effective in domains with continuous state spaces via discretization,  even though the optimality guarantee would be lost there.
%  . We further evaluate on action-pruning and learning performance. 
%  In the worst case, Q-M is equivalent to Q-learning when no actions are pruned.
% We remark that even though computing and maintaining the $Q^{*}$, $Q^{\mu*}$, and $Q_{|R|}^{*}$ for each source behavior %can be expensive either computationally or spatially, 
In general, Q-M
operates under additional \textcolor{black}{computation and space} %\sout{computing resources} 
requirements %\sout{(i.e., CPU time and storage)} 
\textcolor{black}{that are reasonable} to implement in practice, with its benefits \textcolor{black}{to sample complexity} outweighing these costs. %\sout{in practical applications, especially when interacting with the target domain is expensive (e.g., due to wear and tear)}. 
For a comprehensive comparison between Q-M and competing \textcolor{black}{transfer learning} approaches, refer to table \ref{comparison}.

Our core contributions are:
{\color{black}
 We address the problem of reward adaptation (RA) via Q-Manipulation (Q-M) in domains with discrete state and action 
 spaces, and demonstrate that Q-M is provably efficient to sample complexity \textcolor{black}{due to action pruning}, and represents a new approach to RA that supports more general knowledge transfer than the previous work. 
 % \textcolor{blue}{We do not impose any assumptions on similarity between source and target tasks.}
   %\sout {We formally   prove the correctness of the action pruning process under certain initialization conditions; }
    \textcolor{black}{
    We introduce two methods \textcolor{black}{that both leverage a lite-model for capturing neighboring-state information}:
    1) Q-Manipulation (Q-M), which modifies Bellman updates to compute upper and lower bounds on the target domain's $Q^*$ %\sout{using only a lite-model},
    \textcolor{black}{in a value-iteration 
    % like way
    style update} 
    and 2) Monotonic Q-Manipulation (M-Q-M), which extends Q-M by incorporating source Q-functions \textcolor{black}{for initializing the bounds and then iteratively tighten them.} %\sout{to derive tighter bounds.} 
    %under specific initialization conditions. 
    We formally prove the correctness of \textcolor{black}{both methods}. %\sout{the action pruning procedure using the bounds derived from either method.}
     % We introduce modified Bellman updates that compute upper and lower bounds on the target domain's $Q^*$, leveraging a lite-model only.
     % , which is insufficient for full planning. 
     % Additionally, we extend this approach to Monotonic Q-M (M-Q-M) and derive tighter bounds (using source Q functions) under specific initialization conditions and formally prove the correctness of the action pruning procedure.
     Since optimal actions \textcolor{black}{are preserved} 
     % \sout{is never pruned}
     , there is no negative transfer.
    }
     % otherwise, we suggest how Q-M may be applied to expedite learning at the cost of guaranteed optimality. 
    %and further show that it leads to convergence. 
    % \color{red}In deterministic MDPs, Q-M is equivalent to value iteration and 
    % %Given the function that relates source reward function to target reward function, 
    % becomes zero-shot.}
    % We also, prescribe Q-M for potential target reward combination functions which can prune out a larger number of actions under certain assumptions.
    We extensively evaluate Q-M against baselines under its theoretical assumptions to validate its efficacy and analyze its limitations, \textcolor{black}{such as with nonlinear and imperfect combination functions.}
    %\sout{We also illustrate how some of the assumptions may be relaxed \textcolor{black}{for M-Q-M}, such as extending it to address a nonlinear combination function 
    % allowing noise in the combination function at the cost of losing optimality guarantees, or allowing noise in the combination function, at the cost of reduced pruning (lower boost in performance).}
    \textcolor{black}{Results confirm Q-M as a valuable approach for RA. }
    % To demonstrate that the idea would also benefit continuous state spaces, we introduce a preliminary process to handle such domains and present comparison results. 
    % Finally, we present a case study to illustrate, in practice, 1) how Q-M would be applied as a modular approach, and 2) how it may be combined with its competing approaches to further enhance transfer efficiency (revealing the unique mechanism Q-M takes in contrast to the priors approaches for RA). 
    % Finally, we illustrate how these assumptions may be relaxed in practice in real-world problems via a simulated domain of real-world relevance.
    % To the best of our knowledge, other than approaches that compose complex behaviors from existing behaviors, such as in hierarchical and composable RL~\cite{simpkins2019composable,doroodgar2010hierarchical} (see related work), little effort has been invested to exploit these available behaviors.
    %When individual behaviors are not subgoals to target, an approach like \cite{barreto2018successor} may not perform efficiently. 
    % Q-M may be viewed as a special form of transfer learning with a focus on manipulation the Q functions and 
    % simultaneously a special form of modular RL that maintains and combines knowledge from source domains to benefit the learning of novel target behaviors.
    % %generate novel behav
    % %adapting the reward function. 
    % Thus, our work contributes a unique perspective to achieving efficient RL.
 }   

   % We propose Q-Manipulation for reward adaptation to expedite learning the target behavior given the source behaviors. It reuses knowledge from the source behaviors in an indirect way during learning. We prove that our approach does not affect the optimality of the returned policy and thus guarantees improvements in sample complexity. Our evaluations provide further insights into the effectiveness of  Q-Manipulation with respect to several baselines and its ability to generalize to various reward adaptation settings. 
    %As a result, we can improve performance compared to learning from scratch.
\section{Related work}
\textcolor{black}{Transfer learning and multi-task learning have emerged as two central paradigms in RL for improving sample efficiency and generalization. This section surveys existing approaches across these domains, with a focus on how prior experience from one or more source tasks can be exploited to accelerate or stabilize learning the target task.}
% \textcolor{red}{How about moving the table here?}

\textbf{Transfer Learning in Reinforcement Learning:}
The goal of transfer learning in RL is to utilize knowledge gained from previously solved tasks (source tasks) to improve performance in a different, typically unseen task (target task). Foundational surveys such as \cite{taylor2009transfer, wulfmeier2023foundations} categorize transfer techniques by how the knowledge is transferred, ranging from policies and value functions to learned models and internal representations. 
% Since Q-M naturally fits within the landscape of transfer learning when the reward function is altered between tasks, our focus is specifically on prior work addressing such reward-based task variations.
% , and how transfer is operationalized.
 According to the taxonomy of transfer reinforcement learning proposed by \cite{taylor2009transfer}, %\sout{Q-M transfers knowledge using Q-values where the allowed task difference is $r$. We do not prescribe a task selection mechanism, but the task is selected based on the combination function ($f(R_i)$).}
 \textcolor{black}{reward adaptation belongs to transfer learning where the allowed task difference is $R$. }
 % As we would like to evaluate sample complexity we rely on comparision based on target time to threshold using any RL method.
 %\sout{Since Q-M naturally fits within the landscape of transfer learning when the reward function is altered between tasks, our focus is specifically on prior work addressing such reward-based task variations.}
 \textcolor{black}{Even though more general transfer methods can be applied to reward adaptation, such as~\cite{mann2013directed}, these methods are often heuristic in nature due to the  generic relationship between the source and target domains. }
 Within the category of reward adaptation in transfer RL, several relevant methods have been proposed. Notably, SFQL and SQB represent recent efforts aimed at improving transfer efficiency through the use of successor features ($\phi$ and $w$) and action-value bounds, respectively. 
 % Traditional Q-learning (QL) is also commonly used as a baseline to assess potential negative transfer effects in such settings.
% in our evaluation to understand sample complexity using the target time to threshold. 
\textcolor{black}{
Table \ref{comparison} presents a comprehensive comparison of three methods, detailing their key characteristics, including assumptions, core strategies, and intended use cases to provide a clearer understanding of the purpose and design of each approach.
}
% \sout{In addition, \cite{mann2013directed} presents a more general form of transfer that leverages weak admissible heuristics to guide exploration by prioritizing promising actions. For reward adaption, their problem is simplified but the transfer is limited to a single source and target domain. In contrast, Q-M can transfer from multiple source behaviors to a target and reduces the effective search space via action pruning.  \cite{mann2013directed} still requires exhaustive exploration and may be sensitive to discrepancies in reward structure (risking loss of optimal choice). 
% % We include this work as a representative for provably efficient transfer 
% % , though its applicability is limited by its single-source setting and reduced robustness to large differences in source and target objectives (optimal solution may be lost).
% }
% \renewcommand{\arraystretch}{1.4}
\begin{table}[!htbp]
\centering
\footnotesize
\begin{tabular}{|p{2.5cm}|p{2.1cm}|p{2.9cm}|p{1.8cm}|p{3.1cm}|}%{|l|l|l|l|l|}
\hline
\multicolumn{2}{|c|}{} & \textbf{ Q-M (M-Q-M)}   & \textbf{SFQL \cite{barreto2018successor}} & \textbf{SBQ \cite{adamczyk2024boosting}} \\
\hline
\multirow[t]{2}{*}{\textbf{Assumptions}} & Source Domain & \textcolor{black}{$R_i$, $\hat{T}(\cdot\mid s,a)$}, Q-variants (M-Q-M only) & $\phi$ and $w_i$  & Initial Q \\
\cline{2-5}
 & Target Domain & $f$ %\sout{and $\hat{T}(\cdot \mid s,a)$}
 & $w_{target}$ & NA \\
\hline
\multicolumn{2}{|l|}{\textbf{Strategy for transfer}} & Action pruning & Warm start & Clipped Bellman update \\
\hline
\multicolumn{2}{|l|}{\textbf{Guaranteed improvement on sample efficiency}} & \textbf{Yes} & No & No \\
\hline
\multicolumn{2}{|l|}{\textbf{Pre-Learning or Online Learning}} & Pre-Learning & Pre-Learning & Online Learning \\
\hline
\multicolumn{2}{|l|}{\textbf{Modularity}} & Yes & Yes & Yes \\
\hline
\multicolumn{2}{|l|}{\textbf{Robust against Negative Transfer}} & \textbf{Yes} & No & No \\
\hline
\end{tabular}
\caption{Comparative analysis of Q-M, SFQL and SQB for transfer learning given source domain Q-values, target reward as combination of source reward $\mathcal{R}=f(R_i)$ and lite-model
% when $R_i$ differs across target task. 
$\hat{T}(\cdot \mid s,a)$}
\label{comparison}
\end{table}

\textbf{Multi-Task Reinforcement Learning (MTRL):}
MTRL addresses transfer from a different angle by aiming to jointly learn across multiple tasks. Rather than assuming task isolation, MTRL leverages shared experience across a distribution of tasks to improve generalization or reduce training time on each task \cite{vithayathil2020survey}. Common methods involve parameter sharing at the level of policy networks, value functions, or representation encoders. For example, \cite{d2019sharing} explores decentralized training where multiple agents share parameters with a central model during training, facilitating cross-task generalization.
% While multi-task reinforcement learning (MTRL) methods can be implemented in both online (learning multiple tasks in parallel) and offline (sequential task learning) settings~\cite{vithayathil2020survey, yu2020meta}, many existing approaches assume task similarity and are optimized for joint or parallel training. 
MTRL typically involves joint learning across tasks.
Moreover, MTRL approaches typically rely on strong assumptions about the alignment of tasks during training, such as tasks belonging to the same distribution~\cite{taylor2009transfer}.
These assumptions are necessary for stable joint optimization but significantly limit robustness when attempting to generalize or transfer across structurally diverse or highly heterogeneous tasks.
% , but several methods also support sequential or continual learning settings \cite{vithayathil2020survey}. 
%\sout{Given the scope of this paper, we are interested in \textcolor{brown}{order-independent} \sout{sequential or continual} learning settings \textcolor{brown}{for source domains} where tasks differ in terms of reward function without any assumption on the distribution of the tasks.}
% This can make them less effective when tasks are highly heterogeneous or when the target task must be addressed post hoc without concurrent access to source tasks~\cite{teh2017distral}.

Despite their promise, \textcolor{black}{existing} transfer RL and MTRL methods often assume a high degree of behavioral or reward similarity between source and target tasks. When this assumption does not hold, transferred knowledge can mislead the learning process, a phenomenon known as \textit{negative transfer}. To mitigate this, some approaches investigate mechanisms for estimating or bounding task similarity prior to transfer, though such measures are not always practical, reliable, or easy to compute \cite{carroll2005task, taylor2009transfer}.
% \textcolor{red}{citations needed here.}
\textcolor{black}{The proposed approach to reward adaptation does not rely on this assumption and thus bridges an important gap in transfer learning, offering a complementary and robust strategy for leveraging prior task knowledge across tasks, regardless of their similarity.}

%\sout{The Q-M approach introduced in this work diverges from traditional assumptions in both transfer and multi-task learning. Rather than requiring behavioral or reward similarity across tasks, Q-M operates in a \textit{continual learning setting}. 
% where a set of pre-trained source Q-functions, potentially diverse, are available. By reasoning over these Q-functions post hoc, Q-M selectively prunes provably suboptimal actions in the target task, enabling safe and efficient value-based transfer. The pruning mechanism allows Q-M to harness prior experience without requiring fine-tuning, shared parameterization, or reward structure alignment. In doing so, it fills an important gap in the transfer learning literature, offering a complementary and robust strategy for leveraging prior task knowledge across dissimilar tasks. Additionally, prior work, such as \cite{zhuo2017model}, has explored the use of lite-models for planning. Lite models are incomplete or approximate models of the environment that provide limited or coarse-grained information about state transitions. These models are sufficient to support planning or decision-making in the absence of a full, detailed model, often by relying on abstractions, heuristics, or partial knowledge. However, transfer RL methods have predominantly relied on sample-intensive model-free or model-based approaches. Q-M, on the other hand, enables transfer with a lite-model, which is not adequate for planning but enables action pruning resulting in efficient transfer.}

\textcolor{black}{
\textbf{Other Related Paradigms:}
Several adjacent fields intersect with the goals of transfer in RL and RA. Reward decomposition \cite{russell2003q, van2017hybrid} breaks down complex reward signals into composable sub-rewards to simplify learning. Multi-objective RL \cite{roijers2013survey, vamplew2011empirical} optimizes policies that balance trade-offs among multiple predefined objectives. Hierarchical RL (HRL) \cite{dietterich1998maxq, bacon2017option} structures learning around temporal abstraction, using high-level policies to control low-level sub-policies. While conceptually related, these approaches often rely on known task structure, explicit sub-goals, or carefully designed reward factorizations, which limit their applicability to general transfer scenarios or in some cases they are only scaling up RL instead of effective transfer of knowledge.
}

\section{Proposed approach}
% \vskip-5pt
%We address the RA problem using \textit{Q-Manipulation} approach in RL framework. 
In this section, we start with a brief introduction to reinforcement learning (RL) before discussing reward adaptation (RA) and formulating our approach. 

\subsection{Preliminaries}
% In the technical discussion of Q-M, we focus mostly on discrete state and action spaces, and then extend to continuous state spaces towards the end. 
% Extension to continuous action spaces will be addressed in our future work. 
In RL, the task environment is modeled as an MDP
  $M = (S,A,T,R,\gamma)$, where $S$ is the state space, $A$ is the action space, $T: S \times A \times S \rightarrow [0,1]$ is the transition function, $R: S \times A \times S \rightarrow \mathbb{R} $ is the reward function, and $\gamma$ is the discount factor. At every step $t$, the RL agent observes state $s_{t}$ and takes an action $a_t \in A$. As a result, the agent progresses to state
$s_{t+1}$ according to the transition dynamics $T(\cdot | s_{t}, a_{t})$, 
and receives a reward $r_{t}$.  The goal is to search for a policy that maximizes the expected cumulative reward. 
% \sout{We use $\pi$ to denote a policy as a mapping from $S$ to $A$.} 
We use $\pi: S \to A$ to denote a policy.
% , which is a mapping from the state space $S$ to the action space $A$.
The $Q$ function of the optimal policy $\pi^*$ is denoted by $Q^*$ and defined in Eq. \ref{1}.

% \begin{wrapfigure}{r}{0.4\textwidth}
%     \begin{equation}\label{1}
%         Q^*(s,a) = \max_{\pi} \left[ \mathbb{E} \left[ \sum_{t=0}^{\infty} \gamma^t r_t | s_0, \pi \right] \right]
%     \end{equation}
%     \begin{equation}\label{2}
%          Q^{\mu}(s,a) = \min_{\pi} \left[ \mathbb{E} \left[ \sum_{t=0}^{\infty} \gamma^t r_t | s_0, \pi \right] \right]
%     \end{equation}
% \end{wrapfigure}
\begin{wrapfigure}[5]{r}{0.4\textwidth}
  \vspace{-28pt} % reduce space above
  \begin{minipage}{\linewidth}
    \begin{equation}\label{1}
        Q^*(s,a) = \max_{\pi} \left[ \mathbb{E} \left[ \sum_{t=0}^{\infty} \gamma^t r_t | s_0, \pi \right] \right]
    \end{equation}
    \begin{equation}\label{2}
         Q^{\mu}(s,a) = \min_{\pi} \left[ \mathbb{E} \left[ \sum_{t=0}^{\infty} \gamma^t r_t | s_0, \pi \right] \right]
    \end{equation}
  \end{minipage}
  % \vspace{-10pt} % reduce space below
\end{wrapfigure}
To prepare us for later discussion, we also introduce $Q^{\mu}$ (Eq. \ref{2}) to represent the $Q$ function of the ``worst'' policy that minimizes the expected return.  %in contrast to $Q^*$ which corresponds to the maximum expected return. %Furthermore, 
The following lemma establishes the connection between $Q^{\mu}$ and a variant of $Q^*$: %(ensuring additional computation can be ignored).
%Throughout the paper, proofs, if omitted, are included in the appendix. 
%\begin{lemma}\label{lemma1} 
%\begin{equation}\label{projection}
 %   {Q^{\mu}_{R}}(s,a) = {-Q^{*}_{-R}}(s,a)
%\end{equation}
%where ${Q^{*}_{-R}}(s,a)$ denotes the Q function of the optimal policy under negative $R$ or $-R$. 
%\end{lemma} 
\begin{lemma}\label{lemma1}
$Q^{\mu}_{R}(s,a) = -Q^{*}_{-R}(s,a)$, where $Q^{*}_{-R}(s,a)$ denotes the Q function of the optimal policy under negative $R$ or $-R$.
\end{lemma}
% \vskip-8pt
In this paper, we consider RL with discrete state and action spaces and deterministic policies. Extending the discussion to the continuous cases and stochastic policies will be future work. Proofs throughout the paper are in the appendix. 

\subsubsection{Reward Adaptation (RA)}
% \begin{definition}[Reward Adaptation (RA)] Under $M \setminus R$, denoting an MDP without the specification of a reward function, RA is to determine the optimal policy for a target reward function $\mathcal{
% {R}}$, given a set of source behaviors trained under their respective source reward functions ${
% {R}}_1, {
% {R}}_2 \ldots {
% {R}}_n$.
% \label{def:ra}
% \end{definition}
% \vskip-12pt
In RA, we \textcolor{black}{consider only the adaptation of the reward function, and hence} assume the same transition dynamics, state, and action spaces for the source and target behaviors. 
\textcolor{black}{Furthermore, the source and target reward functions are assumed to be related, e.g., via success features in~\cite{barreto2018successor}.}  
% The source domains \textcolor{black}{can pass on certain information but} will no longer be accessible when learning the target behavior. 
% \textcolor{red}{Can we consider:}
The source domains are no longer accessible while learning target behaviors but may \textit{transfer
knowledge} to the target domain.
% \textcolor{red}{Need to specify why this matters?:} 
%\sout{\textcolor{brown}{Therefore, alternating between source and target domains during training is not possible, which introduces additional challenges, as it necessitates the extraction and retention of transferable knowledge that can be effectively transformed and applied to the target task.}}
Next, we provide the problem statement of RA under our approach as follows:

\begin{definition}{Reward Adaptation}: %\sout{with $Q$-Variants}]: 
\textcolor{black}{RA is a problem to }%\sout{Given an RA problem where variants of the $Q$ functions  
%$Q^{*}_{i}$ and $Q^{\mu}_{i}$  %$Q^{*}_{i_{|R|}}$ are accessible for the source domains (e.g., $Q^*$'s and $Q^{\mu}$'s under the source reward functions),} 
%for each source domain indexed by $i$,}
determine the optimal policy under a target reward function $\mathcal{R}$ that is \textcolor{black}{under a \textbf{known} functional relationship with }
% \sout{a known function of} 
the source reward functions specified as follows: %$f$ of the source reward functions:
% \vskip-12pt
\begin{equation}
\mathcal{{R}}= f({R}_1, {R}_2, \ldots {R}_n) %\text{\;\;where\;\;} %i \in n
\label{combination}
\end{equation}
where $f$ is  referred to as the combination function that relates rewards from the source and target domains.
\end{definition}

\textcolor{black}{The above definition provides a general formulation of RA when no restriction is placed on $f$, equivalent to that in~\cite{barreto2018successor} when each state is viewed as a unique feature. In practice, however, we often restrict the form of $f$, such as assuming linearity.
In such cases, even though the above formulation becomes less expressive than ~\cite{barreto2018successor} for RA problems, it enables Q-M to drop the requirement on the similarity between the source and target behaviors.
When $f$ must be learned, $f$ can also be augmented with a noise component to accommodate for imperfect observations or function mappings (more discussion later).}
% \sout{When $f$ is not known exactly but can be modeled with an additional noise component,
% we will discuss later how Q-M can be adapted to handle such situations at the cost of reduced efficacy. }
% \sout{This target reward formulation aims to facilitate a different way of ``composing''  learned skills to generate new skills. Traditional approaches~\cite{barreto2018successor, dietterich1998maxq} often rely on logical operators to combine behaviors, necessitating prior knowledge about the logical relationships between these behaviors. In contrast, our approach enables transfer through functional relationships of the reward functions, thereby broadening the flexibility of such composition.} %and adaptability of the learned policies.}

% where $\mathds{R}(s, a, s') = \mathbb{E}[\mathcal{R}(s, a, s')]$ and $\mathds{R}_{i}(s, a, s') = \mathbb{E}[R_i(s, a, s')]$ under noisy rewards. 
%is the expected reward function for the $i^{th}$ source domain. 
%\label{def:bra}
%not explored and thereby makes the learning sample efficient even before starting Q-learning. 

\subsubsection{Provably Efficient   Transfer in Q-Learning} \label{SC}
A central objective in transfer learning is to establish theoretical guarantees that knowledge transfer leads to improved sample efficiency \citep{agarwal2023provable, tirinzoni2020sequential, mann2013directed}. 
\textcolor{black}{However, these results do not apply readily to reward adaptation, which has a more restrictive problem setting that allows the analysis to be more targeted. In particular, in our approach, the efficiency 
% \sout{to} 
\textcolor{black}{of}  sample complexity lies in action pruning.}
Using Theorem~7 from \cite{qu2020finite}, the sample complexity bound is given by  
$
\mathcal{O}\!\left(\frac{(|\mathcal{S}||\mathcal{A}|)^2 \, t_{\mathrm{mix}}}{(1-\gamma)^5 \, \varepsilon^2}\right),
$
where $|\mathcal{S}|$ denotes the cardinality of the state space, $|\mathcal{A}|$ denotes the cardinality of the action space, $t_{\mathrm{mix}}$ is the mixing time of the underlying Markov chain induced by the policy, $\gamma \in (0,1)$ is the discount factor, and $\varepsilon > 0$ is the accuracy parameter.  
% Note that the dependence on $(1-\gamma)^{-5}$ reflects the contribution of the effective horizon to the sample complexity \textcolor{red}{why?}. 
% \textcolor{green}{Our readers do not know about pruning yet, as we haven't introduced it in our approach} \textcolor{blue}{you did in the intro}
Q-M pruning strategy eliminates suboptimal actions while ensuring that the optimal action is preserved, thereby yielding a reduced action set $\tilde{A} \leq \mathcal{A}$. 
\textcolor{black}{
Furthermore, pruning may render some states of the original MDP unreachable, since transitions associated with pruned actions are removed. 
This is equivalent to cutting down the state space since the theory applies to space that is ergodic. 
}
% \sout{However, pruning may render some states of the original MDP unreachable, since transitions associated with pruned actions are removed. As a result, the entire state space $\mathcal{S}$ need not remain ergodic under $\tilde{A}$. To account for this, we restrict attention to the subset of states that remain reachable under the reduced action set, which induces a sub-MDP that is itself ergodic.
% % This ergodic sub-MDP is the object of our sample complexity analysis.  
% Under the ergodicity assumption, the sub-MDP may have both a reduced number of reachable states and a smaller action set.} 
Hence, action pruning
directly impacts the $(|\mathcal{S}||\mathcal{A}|)^2$ term in the bound. Moreover, the simplified transition structure typically leads to a smaller mixing time $t_{\mathrm{mix}}$.  
Together, these effects yield a \emph{polynomial}  (at least a quadratic)
% \textcolor{red}{at least a quadratic?}
reduction in sample complexity, as compared to the original MDP.
\textcolor{black}{It is also worth noting that action pruning not only benefits sample complexity but also regret. 
Even though these are separate concepts in RL, 
they are often related: the theoretical regret bounds often depend polynomially on the cardinality of the state and action space  (\cite{zhang2020almost, bai2019provably}).}
\subsection{Q-Manipulation} \label{QM-iteratio}
\vskip-3pt

\textcolor{black}{In transfer learning, the source domains can pass information to the target domain to facilitate its learning. In RA such as ~\cite{barreto2018successor}, for example, what is passed includes 1) successor features, which are essentially discounted feature counts, and 2) source weights. The successor features of a source policy allow it to be evaluated easily under any new task given its weights.  In Q-M, we assume the following to be passed: a) source reward function, and b) a lite-model of the environment (denoted by $\hat{T}(\cdot |s,a)$), where the lite-model captures neighboring information only. For example, the neighbors of a state after executing an action may be represented as a small region around that state. Since the lite-model does not model the distribution of such neighbors, it is much easier to specify or learn (while learning the source behaviors) than the full dynamics model in MDP. 
For the target domain, feature weights are assumed in ~\cite{barreto2018successor} and $f$ is assumed in Q-M. 
Both may also be learned. 
From this perspective, Q-M is comparable to ~\cite{barreto2018successor} in terms of resource demands.}

In Q-M, we 
% first initialize an upper and lower bound of $Q_{\mathcal{R}}^*$ and then 
iteratively refine an upper and lower bound (UB and LB) of $Q_{\mathcal{R}}^*$ . 
% We refer to this kind of update as Q-M-Vanilla(Q-M-V). 
% \sout{To avoid notation cluttering to improve clarity, we omit the subscript of $Q$ for indicating the reward function used. }
These 
% two 
steps are formalized below:
%We propose a solution to prune actions based on upper and lower bound Q-values for the target domain. It is challenging to prune actions from loosely defined bounds so we introduce a refinement process similar to value iteration:

\color{black}
% \begin{align}
%     Q_{k+1}^{UB}(s, a) &= \max_{s' \in {\hat{T}(\cdot |s,a)}} \left[ \mathcal{R} (s, a, s') + \gamma \max_{a'} Q(s', a') \right] \\
%     Q_{k+1}^{LB}(s, a) &= \min_{s' \in {\hat{T}(\cdot |s,a)}} \left[ \mathcal{R}(s, a, s') + \gamma \max_{a'} Q(s', a') \right]
% \end{align}
\begin{definition}[Q-M Bellman Operators]
The Bellman operators for UB and LB in Q-M are mappings $\mathcal{T}: \mathbb{R}^{|S\times A|} \rightarrow \mathbb{R}^{|S\times A|}$ that satisfy, respectively:
% \textcolor{red}{Why do you need the ' in  $\mathcal{T}_{UB}'$ given that you already have UB and LB that distinguishes them from the traditional?}
\begin{align}
    (\mathcal{T}_{max} Q^{UB}_k)(s, a) &= \max_{s' \in {\hat{T}(\cdot |s,a)}}  \left[ \mathcal{R}(s, a, s') + \gamma \max_{a'} Q^{UB}_k(s', a') \right]
    \label{ub_new}
    \\
    (\mathcal{T}_{min} Q^{LB}_k)(s, a) &= \min_{s' \in {\hat{T}(\cdot |s,a)}}  \left[ \mathcal{R} (s, a, s') + \gamma \max_{a'} Q^{LB}_k(s', a') \right]
    \label{lb_new}
\end{align}
\end{definition}
\color{black}

\textcolor{black}{More specifically}, $\hat{T}(\cdot |s,a)$ denotes {\textbf{1-step reachable states}} from $s,a$.
% \sout{This information is assumed to be available in Q-M or can be obtained while training the source behaviors. }
\textcolor{black}{In practice, $\hat{T}(\cdot |s,a)$ can be estimated via memorization while learning the source behaviors. Such information can then be consolidated by the target domain. An assumption here is that each transition must have been experienced at least once by at least one of the source domains during learning, which is a much less stringent requirement than that in the value convergence of RL.
When information regarding $\hat{T}(\cdot |s,a)$ is available, it can be provided directly to the target domain without 
implicit knowledge transfer.
% information passing. 
A similar practice can be adopted for the source reward functions, which are combined via $f$ to compute the target reward function used above (i.e., $\mathcal{R}$).}
\textcolor{black}{With deterministic domains, the Q-M Bellman operators above are exactly the operator in value iteration. }

\textcolor{black}{Note about lite-models: \cite{zhuo2017model} has explored the use of lite-models for planning.
Lite models are incomplete or approximate models of a domain: they provide limited or coarse-grained information about the domain dynamics.
While the lite-model considered in our work cannot be used directly for planning, it can support reinforcement learning, opening up a new possibility for transfer learning. }

Similar to value iteration, 
the UB and LB can be initialized arbitrarily, established by the following theorem. 
% \sout{Learning a lite-model, i.e., an incomplete model of the environment, has been shown to be considerably more tractable and can still yield plans with a high probability of success \cite{zhuo2017model}. In contrast, learning a complete and accurate world model remains a challenging and largely open research problem. Full model-based planning suffers from compounding errors over time steps due to imperfect model predictions. Focusing on 1-step reachability avoids this issue entirely. The 1-step reachable set is often binary or categorical: Is state s' reachable from (s,a) or not? This is much easier to provide or learn than modeling probability distributions over next states, especially in stochastic environments. As a consequence, we advocate for the use of $\hat{T}(\cdot |s,a)$ as a lite-model. While $\hat{T}(\cdot |s,a)$ is insufficient for full-fledged planning, it enables the derivation of double-sided bounds on the target action-value function $Q^*$.} 

% \sout{Similarly, the source reward functions or ${R}_{i}$'s are also assumed to be available \textcolor{red}{(may be memorized with a single observation in source domains in practice)} so that $\mathcal{R}(s,a,s')$ in the equations above can be computed based on its known relationship with them (Eq. \ref{combination}). Next we prove that such an update converges to a unique fixed point.}

\color{black}
\begin{theorem}[Q-M Convergence]\label{qms} 
$\mathcal{T}:\mathbb{R}^{|S\times A|} \rightarrow \mathbb{R}^{|S\times A|}$ is a strict contraction
such that the $Q$ function  converges to a unique fixed point for UB and LB, respectively, or more formally:
\[
\begin{aligned}
\|\mathcal{T'} Q_k - \mathcal{T'} Q_{k+1}\|_{\infty} 
&\leq \gamma \|Q_k - Q_{k+1}\|_{\infty}, 
\forall Q_k, Q_{k+1} &\in \mathbb{R}^{|S \times A|}
\end{aligned}
\]
% \vskip-9pt
\end{theorem}

% \begin{theorem}
% For a discount factor $\gamma \in [0, 1)$, the operators $T_{UB}'$ and $\mathcal{T}_{min}$ are contraction mappings on the space of bounded Q-functions equipped with the infinity norm, $\|{\cdot}\|_{\infty}$.
% \end{theorem}
$\|f\|_{\infty} = \sup_{x} |f(x)|$ and hence $\|Q_k - Q_{k+1}\|_{\infty}$ returns the maximum absolute difference between $Q_{k}(s, a)$ and $Q_{k+1}(s, a)$ under any $s, a$ above.
% \sout{Since the theoretical properties for $\mathcal{T'}_{UB}$ and $\mathcal{T'}_{LB}$ are same, we ommit the subscript.}
Given that the Q values for UB and LB may be initialized arbitrarily,
% we \sout{must} \textcolor{orange}{next} 
Next, we
demonstrate that the bounds are
% \sout{preserved} 
valid upon convergence. This is established by the following theorem:
% As the Q-value can beinitialized in arbitrary manner, we need to enprove that after convergence definition of bounds holds.
\begin{theorem} \label{ordering}
Given the standard Bellman operator $T$, $\mathcal{T}_{min}$, and $\mathcal{T}_{max}$ (discussed above),
if the initial Q-functions are ordered such that $Q_0^{LB} \le Q_0 \le Q_0^{UB}$, then
$$
Q_k^{LB} \le Q_k \le Q_k^{UB}
$$
holds for all $k\geq0$. Thus, the fixed points $Q^*$, $Q^{LB}$, and $Q^{UB}$ produced by the respective operators,
satisfy the same ordering in the limit as $k \to \infty$.
\end{theorem}

% \color{black}

% \textcolor{red}{I am not sure the above theorem is what you meant to say. Do you want to say that the equalities hold for all step k, so it should be $\textcolor{orange}{Q_k^{LB} \le Q_k \le Q_k^{UB}}$ instead?}
% \textcolor{blue}{yes when $k\to \infty$}
% \sout{
% It is essential to continue the iteration process until convergence to ensure that bounds satisfy $Q^{LB}\leq Q^{*}\leq Q^{UB}$.
% We can derive the best utility from these bounds when they are tighter. 
% As the number of one-step reachable states increases, the resulting conservative bounds may become loose, thereby limiting their effectiveness in enabling meaningful pruning, which in turn affects sample efficiency.
% So we further improve the bounds by utilizing additional information available from the source Q-value functions.
% }

\subsubsection{Monotonic Q-Manipulation (M-Q-M)}

\textcolor{black}{We have shown that Q-M iteration process will converge to a fixed UB and LB, respectively. However, 
the updates are conservative in the sense that
only the best or worst transitions are considered. 
This can lead to loose bounds even when we start with tighter bounds as initialization, which is counterproductive. 
Next, we consider new update rules to avoid the issue when we have a valid UB and LB to start with. }

\textcolor{black}{First, we formalize the new updates as follows:}
 % \sout{We propose a 2-step (initialization and iteration) update, which is monotonic in nature, given by the following (\textcolor{black}{Note that we do not assume any knowledge of $Q^*$}):}
 
\textbf{Upper Bound (UB)}
\begin{equation}
    Q_{0}^{UB}(s,a) > Q^{*} \text{\;\;\;\;\;\;[Initialization]}
    %= Q_1^{*}(s,a) + Q_2^{*}(s,a) \text{\;\;\; [Initialization]}
\end{equation}
\begin{equation}
\begin{split}
    Q_{k+1}^{UB}(s,a) =&\min\bigg(Q^{UB}_k(s,a),
    \max_{s' \in {\hat{T}(\cdot |s,a)}}
    \left[ \mathcal{R}^{}(s,a,s') + \gamma\max_{a'} Q_k^{UB}(s',a')\right] \bigg)
\label{ub}
\end{split}
\end{equation}
\\
\textbf{Lower Bound (LB)}
\begin{equation}
    %Q^{\mu}(s,a) = \max (Q_1^{*}(s,a) + Q_2^{\mu*}(s,a), Q_1^{\mu*}(s,a) + Q_2^{*}(s,a))  
    %Q^{\mu} < 
    Q_{0}^{LB} < Q^{*}
    \text{\;\;\;\;\;\;\;\;\;\;\;\;\;[Initialization]}
\end{equation}
\begin{equation}
\begin{split}
        Q_{k+1}^{LB}(s,a) =& \max\bigg(Q^{LB}_k(s,a),
        \min_{s' \in {\hat{T}(\cdot |s,a)}}
        \left[ \mathcal{R}(s,a,s') + \gamma\max_{a'} Q_k^{LB}(s',a')\right] \bigg)
\end{split}
\label{lb}
\end{equation}

% In contrast to value iteration, it can be observed that $\min_{s' \in {\hat{T}(\cdot|s,a)}}$ in LB and $\max_{s' \in {\hat{T}(\cdot|s,a)}}$ in UB are used as $T$ is not provided. 
\textcolor{black}{This update guarantees that the upper and lower bounds are monotonically non-increasing (for the upper bound) and non-decreasing (for the lower bound). At every iteration, the bounds satisfy the condition: $Q^{LB}\leq Q^{*}\leq Q^{UB}$.}
The outermost max/min ensures $Q^{UB} \geq Q^{*}\geq Q^{LB}$ throughout the iterative processes via simple induction. 
% \sout{It is worth noting that the updates above ensure that the upper and lower bounds are always decreasing and increasing, respectively, as desired, such that the bounds are tightening.}
When the source reward functions are noisy, it requires their min/ max noise to be used in the updates. 
Next, before discussing the initializations, we show that such processes converge 
%is a contraction mapping that 
to a fixed point in M-Q-M.%, respectively.

\begin{definition}\label{bellman-operator} (M-Q-M Bellman Operators) The min and max Bellman operator for UB and LB in M-Q-M are mappings $\mathcal{T}: \mathbb{R}^{|S\times A|} \rightarrow \mathbb{R}^{|S\times A|}$ that satisfy, respectively: %for any $Q_k \in \mathbb{R}^{|S\times A|}$,
$$
(\mathcal{T}_{min} Q^{UB}_k)(s,a)=\min \left(Q^{UB}_k(s,a), 
\max_{s' \in {\hat{T}(\cdot |s,a)}}
\left[\mathcal{R}(s,a,s') + \gamma\max_{a'} Q_k^{UB}(s',a')\right]\right)
$$
%or
$$
(\mathcal{T}_{max} Q^{LB}_k)(s,a)=\max \left(Q^{LB}_k(s,a), 
\min_{s' \in {\hat{T}(\cdot |s,a)}}
\left[\mathcal{R}(s,a,s') + \gamma\max_{a'} Q_k^{LB}(s',a')\right]\right)
$$
\end{definition} 
%Computing UB an LB in Q-M is thus recursively applying the corresponding operator:
%Q-M Value iteration can thus be represented as recursively applying the Bellman optimality operator:
% $
% Q_{k+1}=\mathcal{T} Q_k .
% $
%Next we verify that $Q^*$ is a fixed point of $\mathcal{T}$, i.e., $\mathcal{T} Q^*=Q^*$. We need to show that $\mathcal{T}$ is a contraction mapping.
% \vskip-15pt
Since the theoretical results for the min and max operators are similar, we do not distinguish between them below but provide separate proofs for them in the appendix.

\begin{theorem}[M-Q-M Convergence]\label{th1} 
The iteration process introduced by the Bellman operator in M-Q-M satisfies
% \vskip-11pt
\[
\begin{aligned}
\|\mathcal{T} Q_k - \mathcal{T} Q_{k+1}\|_{\infty} 
&\leq \gamma \|Q_k - Q_{k+1}\|_{\infty}, 
\forall Q_k, Q_{k+1} &\in \mathbb{R}^{|S \times A|}
\end{aligned}
\]
% \vskip-9pt
such that the $Q$ function converges to a fixed point. 
\end{theorem}
% \vskip-8pt
% Formally, $\|f\|_{\infty} = \sup_{x} |f(x)|$ and it returns the maximum absolute difference between $Q_{k}(s, a)$ and $Q_{k+1}(s, a)$ under any $s, a$ above. The process  
 %is a non-strict contraction mapping so it
  This process converges to a fixed point,
 \textcolor{black}{since the difference between two consecutive iterations always decreases}.
 However, it turns out that the fixed point may not necessarily be unique, as with value iteration. 
\textcolor{black}{
\begin{theorem}\label{unique-fixed-point}
The Bellman operator in Q-M specifies only a non-strict contraction in general:
% \vskip-10pt
\begin{equation*}
\left\|\mathcal{T} Q-\mathcal{T} \widehat{Q}\right\|_{\infty} 
     \leq \left\|Q-\widehat{Q}\right\|_{\infty} 
\end{equation*}
% Where $Q^{UB}_k=\mathcal{T} Q^{UB}_k$ and $\widehat{Q^{\text{\;} UB}_k}=\mathcal{T} \widehat{Q^{\text{\;} UB}_k}$ are fixed points.
\end{theorem}
}
% \vskip-10pt
This result is interesting since it identifies another case where non-strict contraction results in a fixed point other than the identity map. 

\begin{corollary}[Non-uniqueness]\label{nonunique} 
The fixed point of the iteration process in M-Q-M may not be unique.
\end{corollary}
 
 %due to the existence of identity mapping. 
 %(refer to Sec. \ref{proofs})
 In our evaluation, we observe that the fixed point found by the M-Q-M iteration process depends on the initialization. 
 % The tightness of bounds hence depends on the initialization of UB and LB. 
Another observation is that the Bellman operator in M-Q-M appears almost identical to that in value iteration when the MDP is deterministic.
In such cases, we observe that Q-M and M-Q-M often results in zero-shot learning when the upper and lower bounds converge to $Q_{\mathcal{R}}^*$.
% \begin{lemma} \label{l2}
% For Deterministic MDPs, Q-M upper and lower bounds converge to the same values.
% \end{lemma}
%Next, we show some initializations which result in tighter bounds and more pruning opportunities.
\subsubsection{Initializing the Bounds}
\label{initialize}
\vskip-5pt

\textcolor{black}{In order to use M-Q-M updates, the user must provide some correct bounds to start with. To relax such a requirement, next, we show how high quality initialization can be automatically computed for restrictive sets of problems. Computing such an initialization requires additional
% information to pass 
knowledge to be transferred
from the source to the target domains, referred to as Q variants ($Q^*$ and $Q^\mu$).}

A simple way to initialize the bounds would be to identify the most positive and negative rewards and compute the sums of their geometric sequences via the discount factor, respectively. 
However, these bounds are likely to be too conservative to be useful since the iteration processes may converge undesirably due to non-unique fixed points. 
Intuitively, we would like the bounds to be tight initially to yield the best results. 
However, computing bounds for the target behavior based on information from the source behaviors only is not a trivial task. 
Next, we show situations where additional assumptions hold such that we can provide more desirable initializations. 
%further initialize tighten the bounds for efficient initialization.
In particular, we will show next how different forms of the combination function  $f$ in Eq. \ref{combination} can affect the initializations. 
%assuming the availability of variants of Q functions.

% \subsubsection{Linear Combination}
\textbf{Linear Combination Function:}
First, we consider the case when the target reward function is a linear function of the source reward functions. 
In such cases, 
if the agent 
maintains 
%$Q^{*}_{i}$ and $Q^{*}_{|R_i|}$ (Optimal Q for $M\setminus|R_i|$) or $Q^{*}_{i}$ 
both $Q^{\mu}_{i}$'s and  $Q^*_{i}$'s
while learning the source behaviors, we propose the initializations as follows. Note that $Q_i^{\mu}$ can be obtained conveniently while learning the source behaviors based on Lemma \ref{lemma1}.
%and maintains them for future use for efficient initialization of Q-M iteration. 
%\textcolor{black}{Our solution to RA is related to a variant of the Q function, which we refer to as Q-min (note the $\min$ operator below),  denoted by $Q^{\mu}$:}
%and Q variant,  denoted by $Q^{*}_{|R|}$, 
%It is fairly easy to compute $Q_{i}^{\mu}$ (lemma \ref{lemma1}) and $Q_{|R_i|}^{*}$ 
%[which is equivalent to $Q^{*}$ where reward function is |R|] 
%when learning individual behavior ($Q_{i}^{*}$).
%using these Q's from the source domains given a general relationship known between the reward functions.
% If $\mathds{R}=\sum c_i\mathds{R}_{i}$ ($c_i\geq0$)
% %additive of the individual reward function of source behaviors,
% a tighter initialization of Q values can be given by lemma \ref{l4}.

\begin{lemma} \label{l4}
When $\mathcal{R} = \sum_{i=1}^n c_i R_i$ with $c_i \geq 0$, the upper and lower bounds of $Q_{\mathcal{R}}^*$ are, respectively, $Q^{UB}_0 = \sum_{i=1}^{n} c_i Q^{*}_{i} $ and $Q^{LB}_0 = \max_i \left[c_i Q_i^{*} + \sum_j c_j Q_j^{\mu} \right]$, where $j \in \{1:n\} \setminus i$.
\end{lemma}
% \begin{lemma} \label{l4}
%     When $\mathcal{R} = \sum c_i {R}_{i}$ where $c_i\geq 0$
%     %a  R_{i} + b R_{j}$ $(a > 0, b > 0)$
%     , an upper and lower bound of $Q_{\mathcal{R}}^*$ are given, respectively, by: 
% \begin{equation*}
%     \begin{aligned}
%         Q^{UB}_0 &= \sum_{i=1}^{n} c_i Q^{*}_{i} \\
%         Q^{LB}_0 &= \max_i \left[c_i Q_i^{*} + \sum_j c_j Q_j^{\mu} \right] \text{\;\;where\;} j \in \{1:n\} \setminus i
%     \end{aligned}
% \end{equation*}
% \end{lemma}

% \subsubsection{Nonlinear Combination}
\textbf{Nonlinear Combination Function:}
{\color{black}
Handling a nonlinear combination is more complicated and
% by the existence of the discount factor. Since the Q functions capture discounted accumulated rewards, 
deriving tight bounds that are guaranteed to be correct is difficult. 
Instead, we propose approximate bounds for a monotonically increasing and positive function $f$ as follows:
}
$Q^{UB}_0 = f(Q^*_{|R_1|}, Q^*_{|R_2|}, \ldots, Q^*_{|R_n|})$ and $Q^{LB}_0 = -f(Q^*_{|R_1|}, Q^*_{|R_2|}, \ldots, Q^*_{|R_n|})$.
% \begin{align}
%     \begin{aligned}
%         Q^{UB}_0 &= f(Q^*_{{|R_1|}},Q^*_{{|R_2|}},\ldots Q^*_{{|R_n|}})
%         \text{\;\;}
%     \end{aligned}
%    % &&
%    \\
%     \begin{aligned}
%         Q^{LB}_0 &= -f(Q^*_{{|R_1|}},Q^*_{{|R_2|}},\ldots Q^*_{{|R_n|}})
%     \end{aligned}
%     \label{nl}
% \end{align}

% \begin{lemma} \label{l3} 
%  Given an absorbing state MDP with an absorbing state $\hat{s}$ and $\gamma=1$, $Q^{*}$ is bounded by the upper and lower bound specified above.
% \end{lemma}
%     \begin{equation}
%         Q^{UB} \geq Q^{*} \geq Q^{LB}
%     \end{equation}
% \end{lemma}
{\color{black}
% When the transition probability to the absorbing state, or  ${\mathcal{T}}(s,a,\hat{s})$, is known and there is no reward after entering the absorbing state, 
% we can ignore the branch transitioning to $\hat{s}$.
% {\color{black}
% Hence, we can modify UB to use $\max_{s' \neq \hat{s}}\left[ (1-{\mathcal{T}}(s,a,\hat{s})) (\mathds{R}(s,a,s')+ \max_{a'} Q_k^{UB}(s',a'))\right]$ and LB to use $\min_{s'\neq \hat{s}}\left[(1-{\mathcal{T}}(s,a,\hat{s})) (\mathds{R}(s,a,s') + \max_{a'} Q_k^{LB}(s',a'))\right]$. } 
%when contraction mapping is not identity. %can be given by:
%\vskip-8pt
Using the bounds above requires the agent to maintain $Q^*_{|R_i|}$'s.
Since these bounds are approximate, they do not guarantee correctness \textcolor{black}{for M-Q-M} in general, 
% Under this assumption, the convergence proof holds similarly: the Bellman error would remain the same or decrease by a factor of $1-\mathcal{T}(s,a,\hat{s})$ in subsequent iterations. 
% Note that when we use an absorbing state MDP to approximate a discount MDP (i.e., by setting ${\mathcal{T}}(s,a,\hat{s}) = \gamma$), 
% %In case $\gamma \neq 1$, 
meaning that actions belonging to the optimal policy may be pruned. 
However, we show that they work well in practice in our evaluation.

}

\subsubsection{Noisy Combination Function}\label{noisy}
%and Continuous State Spaces}
% \vskip-5pt
 {\color{black}
 %\textbf{Noisy Combination Function:} 
 When the combination function is not known exactly but can be modeled with an additional noise component, such that
 $\mathcal{R} = f({R}_1 \ldots {R}_n) + N$, 
 %= \sum_{i}^{n} \mathds{R}_i + N_i$.
 and we know the range of the noise (i.e., $N_{min}$ and $N_{max}$).
 % In this scenario, we only have access to source behaviors (no noise in R) but the target reward might have an additional uniform noise ranging from ($N_{min}, N_{max}$). 
 % To deal with such a problem we need to ensure popper initialization. 
 We can consider such situations by augmenting the $\mathcal{R}(s, a, s')$ in Eqs. \ref{ub} and \ref{lb} with $N_{max}$ and $N_{min}$, respectively. 
 %To address this problem we further modify the 
 We must also update
 the initialization of the bounds using
 %initial value using 
 $Q^{UB} = Q^{UB} + N_{max}\times\frac{1-\gamma^{t_{max}}}{1-\gamma}$ and $Q^{LB} = Q^{LB} + N_{min}\times\frac{1-\gamma^{t_{max}}}{1-\gamma}$, where $t_{max}$ is the maximum steps in an episode.}
 %This ensures that $Q^{LB}\leq Q^{*} \leq Q^{UB}$ and as a result optimal actions are still retained.
Note however that such modifications will likely reduce the efficacy of Q-M. 

%{\color{black} \textbf{Handling Continuous State Spaces:} For domains with continuous state spaces, we resort to using features (e.g., tile-coding) to discretize the state space and then apply the process of Q-M on such a space to prune actions. We can then run any RL method that can handle continuous state spaces (such as Deep Q-Learning) under the reduced action space per each discrete state. Although the optimality guarantee is obviously lost due to the discretization,  we aim to show how effective such a simple adaption can be.  The implementation details are discussed in Sec. \ref{sec:evaluation}.  We will extend Q-M to natively handle continuous state and action spaces in future work.}

\subsection{Action Pruning in Q-M:} \label{pruningT}
% Given that upper and lower bounds are closer after Q-M iteration, we begin pruning out actions. 
% \sout{Intuitively,} 
If an action $a$'s lower bound is higher than some other action $\hat{a}$'s upper bound under a state $s$, then $\hat{a}$ can be pruned for that state. 
This allows us to reduce the action space per each different state, which contributes to
faster convergence (refer Sec. \ref{SC}). 
%The remaining set of actions are referred to as ``Actions of Interest''. 
% This enables faster convergence through efficient exploration. 
\textcolor{black}{For empirical purposes and to avoid numerical instability, we use a threshold ($\Delta$) and prune only if $ Q^{LB}(s,a) - Q^{UB}(s,\hat{a}) \geq \Delta$. 
% Such a solution also helps M-Q-M in a non-linear combination function (where the correctness of the bound is uncertain).
}
\textcolor{black}{When the source domain's Q values are computed using value iteration with a stopping threshold $\epsilon$, $\Delta$ can be set to be $2\epsilon \frac{\gamma}{1-\gamma}$ to ensure that no actions would be wrongly pruned. When the Q values are approximated (such as via Q learning) and $\epsilon$ is unknown, setting $\Delta$ would not be so straightforward and we delay its treatment to future work.}
When the upper and lower bounds are sound, 
the optimal policies are preserved. 

% {\color{red} Need to talk about the threshold variable here.}
% We also prove that using this strategy of action pruning, optimal action is not pruned.
{\color{black}
\begin{theorem} \label{th2}[Optimality]
%Given that $Q_{R_i}$ and $Q_{\left|R_i\right|}$ of each source behavior are accurate, 
For reward adaptation with Q variants, 
%(Def. \ref{def:bra})
 the optimal policies in the target domain remain invariant under Q-M and M-Q-M when the upper and lower bounds are initialized correctly. % optimal convergence of Q-learning for the target behavior.
\end{theorem}
}
% Given knowledge of the target reward function, the initialization can be handcrafted as required. With one of these initialization methods, we go through the iteration process which results in tighter upper and lower bounds. 

% \textcolor{red}{
% With the methodology in place, we now examine how our approach aligns with and diverges from existing paradigms in transfer reinforcement learning, particularly in the context of RA. To this end, we analyze Q-M through the lens of established transfer learning taxonomies and compare it to related methods operating under similar assumptions.
% }
% \section{Analysis of Transfer RL methods for Reward Adaptation}

\section{Evaluation}
\vskip-5pt
\label{sec:evaluation}

\subsection{Baselines}
 The primary objective here is to evaluate the performance of Q-M using the target time to threshold and analyze its benefits and limitations. 
%Since the focus here is on sample complexity, 
% We compare Q-M with SFQL described by ~\cite{barreto2018successor}, 
% the state-of-the-art approach to reward adaptation. 
We compare Q-M against three methods: SFQL \cite{barreto2018successor}, a state-of-the-art approach for reward adaptation; SQB \cite{adamczyk2024boosting}, which clips the Bellman error using a prior $Q$-function to accelerate learning; and standard Q-Learning (QL) without any knowledge transfer as a baseline.
% \sout{Q-M and SFQL initialize learning in different ways to transfer prior knowledge from the source domains, but otherwise both implement Q-Learning (QL) to learn the target behavior. Hence, we also use QL without any knowledge transfer as a baseline. More specifically, }
To initialize learning for SFQL, we evaluate the given source behaviors in the target domain to compute a bootstrap Q-function as described in the generalized policy improvement theorem in \cite{barreto2018successor}. 
\textcolor{black}{
% \sout{In addition to this, we also compare with SQB \cite{adamczyk2024boosting}.}
% , an approach which utilizes bounds to clip Q-value updates that are overly optimistic or pessimistic.
% This approach can utilize prior Q-function to speed up learning through clipped Q updates. 
% Although the clipped Bellman operator in SQB offers intuitive appeal, there is currently no formal proof demonstrating its effectiveness in improving learning performance.
% \sout{In SQB the clipped Bellman operator offers intuitive appeal but currently there is no formal proof demonstrating its effectiveness in improving learning performance.}
% We keep model-free learning consistent across all our baselines, and so 
% \sout{
% % SQB approximates the model or bounds using samples, as the lite model alone is insufficient.
% % SQB and SFQL do not offer any meaningful modifications to leverage lite model and so we stick to their core updates
% SQB approximates the model or its bounds \textcolor{red}{SQB approximates the model?} using samples, since the lite-model alone is insufficient. SQB and SFQL do not have the ability to leverage the lite-model, so no modifications were made.
% }
% SQB computes bounds online, updating them with learned Q-values, whereas Q-M computes bounds offline and uses them for pruning. 
Additional results for the running time taken by the Q-M iteration process are reported in Sec. \ref{results}.
}
%which is used to initialize learning in SFQL. 
% Additional results analyzing Q-M (including where actions are pruned) and running time comparisons are reported in Sec. \ref{results}. 
% \textcolor{black}{Evaluations on continuous state spaces, \textcolor{blue}{using function approximators}, are presented in the Appendix Sec. \ref{results}. 
%Future work will investigate function approximation methods for enhanced performance in such domains.
% }
{\color{black} We keep the hyperparameters for Q-Learning 
%or DQN
the same across the different methods (refer Sec. \ref{hparams}).}

\subsection{Evaluation Design}

Since we are interested in demonstrating Q-M (short for Q-M/M-Q-M unless separately noted) as a more robust knowledge transfer method than SFQL \textcolor{black}{or SQB}, we design the evaluation domains such that the target behaviors are substantially different from the source behaviors in most of them (similar to the situation in Dollar-Euro). 
\textcolor{black}{Designing evaluations this way also provides an opportunity to study negative transfer in transfer learning.} 
Details on how the source and target behaviors are designed are in the appendix. 
For SFQL, \textcolor{black}{policy evaluation of the source behaviors, required to bootstrap target learning, is achieved via value iteration on the target.} 
% \sout{initializing learning by combining the best parts of the source behaviors, is expected to not perform well unless the target behavior happens to be characterized by some combination of the source behaviors.}
% SQB can be applied with any prior Q, with one of the source behaviors being randomly selected. Since source behaviors differ from the target and are influenced by the dynamic nature of the environment, accurately predicting the time required to estimate the bounds for effective clipping and boosting in QL remains challenging. 
% SQB \sout{can} \textcolor{orange}{is implemented to} use \sout{any prior Q} \textcolor{orange}{a randomly chosen source Q function (computed via value iteration) to bootstrap learning.} 
To analyze the theoretical properties of Q-M, we assume access to \textcolor{black}{accurate} lite-models, reward functions of the source behaviors, and Q-variants (only for M-Q-M and computed using value iteration).
\textcolor{black}{In the appendix, we use memorization and learning to estimate these from source domains, which demonstrate comparable performances.}
% \sout{with a randomly chosen source behavior.} 
% \sout{However, due to differences between source and target behaviors, 
% analyzing samples required
% predicting the time 
% for effective clipping and boosting in QL is challenging.
% Consequently, our evaluations also allow for an assessment of the impact of negative transfer across various baselines.}
For \textcolor{black}{M-}Q-M, we use the initializations described in Sec. \ref{initialize}. 

One observation about Q-M is that the computation of UB and LB is affected substantially by the stochastic branching factor (SBF) of a domain, as evident in Eqs. \ref{ub_new}, \ref{lb_new}, \ref{ub} and \ref{lb}. 
SBF here is defined as the maximum number of next states reachable (or with a nonzero transition probability) from any state and action pair
% \sout{, as derived from the lite-model $\hat{T}(\cdot |s,a)$}. 
Intuitively, the less stochastic the domain is, the more the Bellman updates in Q-M resemble that in value iteration: \textcolor{black}{Q-M updates in deterministic domains are exactly value iteration updates, resulting in zero-shot learning.} 
% \sout{(except for the outermost max/min).}
To demonstrate the influence of SBF, for each evaluation domain, we gradually increase its SBF. 
At the same time, the number of reachable states from a given state is allowed to vary and is randomly chosen between 1 and a set SBF. 
We first evaluate with gridworld domains where combination is a linear combination of source rewards.
\textcolor{black}{We also visualize actions pruning in a chosen domain to illustrate its operation.}
%remaining in each state after pruning to better conceptualize the effect of pruning.
% \sout{Finally, to illustrate how Q-M can be applied to address real-world problems, we show how these assumptions can be relaxed in practice without incurring much more cost. 
% Through this evaluation, we also show how Q-M can be unified with other baselines.}
% We first evaluate with simulation and randomly generated domains under linear combination functions and then move on to the more challenging cases of nonlinear and noisy functions. 
% In addition to this, to showcase 
To evaluate the generality, 
% \sout{We also evaluate
% the generality of Q-M,
% % we compare 
% by comparing}
\textcolor{black}{evaluations are further conducted}
with autogenerated MDPs 
% \sout{\textcolor{brown}{(non-grid dynamics)}}
and with linear and non-linear combination functions 
% \sout{can either be linear or non-linear}.
% we also consider randomizing the domains so that we evaluate with 
% 1) given MDP $\setminus R$ and designed rewards, 2) randomized MDP $\setminus R$ and designed rewards, and 3) randomized MDP $\setminus R$ and randomized rewards.
\textcolor{black}{Finally, we study the effectiveness of Q-M under noisy combination functions, which analyzes the situations when the combination functions must be learned but noise can be bounded.} 
% \sout{Moreover, we also evaluate how is pruning affected if the reward function is noisy by using the method described in Sec \ref{noisy}.}

All evaluations are averaged over 30 runs.
\textcolor{black}{In the convergence plots, we indicate the mean with a solid line, and the shaded region represents a 95\% confidence interval.}
More details about the evaluation settings, along with a detailed description of all the domains, including the design of source and target behaviors, are reported in the appendix.

% \subsection{Linear Combination Function}
\subsection{Gridworld and Linear Combination Function}
\begin{figure*} [!htb]
    \centering
    \includegraphics[scale=0.32]{CI_exp1_5.jpeg}
    \vskip -8pt
    \caption{Convergence plots for Dollar Euro (top), Racetrack (mid), and Frozen Lake (bottom).}
    \label{simulation}
\end{figure*}
% \textbf{Given MDP $\setminus R$ and Designed Rewards:}
In this evaluation, we compared Q-M and M-Q-M with the baselines in simulation domains that include Racetrack, 
Dollar-Euro, and Frozen Lake domain \textcolor{black}{with linear combination functions}. 
The convergence plots are shown in Fig. \ref{simulation}. 
In each subfigure, we show the SBF used (labeled at the top).
We observe that M-Q-M converges substantially faster than the baselines in all three domains.
\textcolor{black}{Depending on how many actions are pruned under traditional Q-M, its performance lies between M-Q-M and QL. In Sec. \ref{a-pruning}, we show how action pruning differs between M-Q-M and Q-M, which results in such an effect on convergence.}
However, as expected, the performance of Q-M and M-Q-M are negatively impacted as SBF increases. 
An interesting observation is the performance of SFQL. 
SFQL seems to struggle with these domains, especially racetrack and frozen lake domains.
Since the source behaviors differ much from the target behavior, knowledge transfer in SFQL based on combining the source behaviors can actually misguide the learning process.
It is worth mentioning that SFQL eventually converged to the optimal policy after we allowed it to train with more episodes. 
In addition, we also observe that Q-M
% \sout{(and Q-M)} \textcolor{red}{(already explained at the beginning of evaluation that Q-M is short for Q-M/M-Q-M unless noted separately)}
in deterministic scenarios (leftmost subfigures when SBF = 1) result in zero-shot learning: their iterative processes for computing UB and LB both converge to $Q^*_{\mathcal{R}}$.
\textcolor{black}{Similar to SFQL, SQB also struggles in the racetrack and frozen lake domain.}
This result demonstrates that Q-M is indeed more robust \textcolor{black}{against negative transfer, and thus represents a more general}  knowledge transfer method that does not depend on the similarity between the source and target behaviors. 
% \sout{Thus, Q-M is robust to avoiding negative transfer.}

\subsubsection{Analysis of Action Pruning}\label{a-pruning}
For gridworld domains (with 4 actions), to understand the states where actions are pruned, we plot heat-maps (refer Fig. \ref{heatmaps} \textcolor{black}{for the Dollar Euro domain}). 
{\color{black}
In all three domains, we observe significant pruning around the terminal states. In addition, we also observe that fewer actions are pruned as SBF increases.  
The following color codes are used:
initial state = yellow and goal states = green.
% , terminal states/obstacles = black.
We use different shades of blue to illustrate how many actions are pruned in a state: the lighter the color, the fewer the actions remain.
}
\textcolor{black}{Upon comparing Dollar-Euro domain's action pruning using M-Q-M and Q-M, we observe that Q-M results in pruning fewer actions (as shown in the Fig. \ref{heatmaps-qmv}). The additional information used by M-Q-M is able to anchor the values to better bounds than the unique fixed point identified by Q-M, which results in more pruning opportunities in M-Q-M. As the SBF increases, Q-M prunes out fewer actions, and so performance becomes similar to QL. This trend is consistent across other domains as well.
}
\begin{figure*}[!htb]
    \centering
    \includegraphics[scale=0.36]{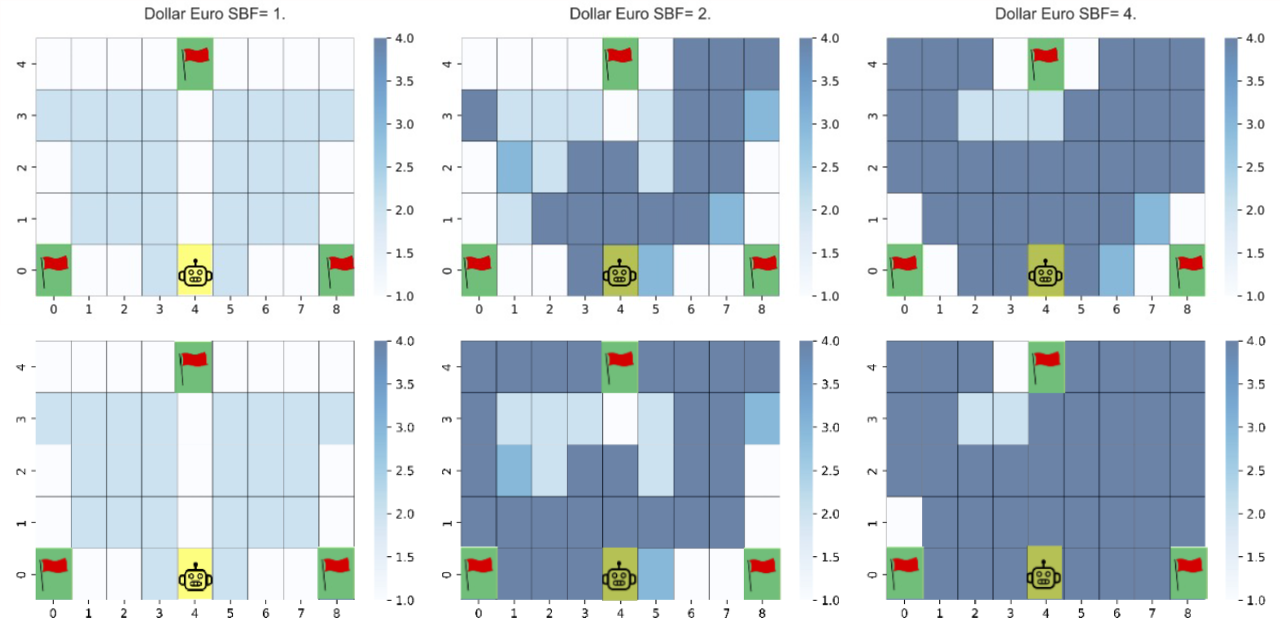}
    \vskip -8pt
    \caption{Heat-maps illustrating action pruning in the Dollar Euro domain using M-Q-M (top) and Q-M (bottom). Lighter shade of blue indicates fewer action remain after pruning.} \label{heatmaps-qmv} \label{heatmaps}
\end{figure*}

\subsection{Autogenerated MDP with Linear and Non Linear Combination Function}
\begin{figure*}[!htb]
    \centering
    \includegraphics[scale=0.32]{CI_exp2_5.jpeg}
    \vskip -8pt
    \caption{Convergence plots for auto-generated domains: $\mathcal{R}=R_1 +R_2$ (top) and $\mathcal{R}=(R_1+R_2)^{3}$ (bottom).}
    \label{linear}
\end{figure*}
% First, we evaluated with the Frozen Lake domain while  randomizing the hole locations ($4$ holes) in each run. Additionally, 
In order to test generalization beyond gridworld and linear combination function,
we evaluated with auto-generated MDP where $T$ is randomly generated in each run.
The terminal states were held fixed as well as their terminal rewards. 
The convergence plots are presented in Fig. \ref{linear}, \textcolor{black}{averaged over 30 different MDPs}.
Similarly, we can observe that M-Q-M performs the best in 
both domains.
% \textcolor{red}{the domain with the linear sum of rewards and Q-M-V outperforms other baselines when the combination of rewards is non-linear.
% The performance of Q-M relies on useful additional information of the bounds. When the combination function is non-linear, the initially computed bounds may not always be helpful. Given different initialization, Q-M may converge to a fixed point different from Q-M-V.%, and as a result, fewer actions are pruned, which results in slower convergence than Q-M-V.  
% }
We also observe that when the combination function is a nonlinear combination, M-Q-M and Q-M's performance drops due to reduced action pruning.
It demonstrates that our methods can generalize to MDPs beyond gridworld dynamics and different target reward combinations.

\subsection{Noisy Combination Function}
We aim to evaluate how Q-M would perform under noisy combination functions and how noise affects its performance.
\textcolor{black}{
We used the same setting as the autogenerated MDP described above. 
We consider a situation where the combination function is not exactly known but can be modeled by using a noise component: $ \mathcal{R} = R_1 + R_2 + N$.
%As a result $f$ is not exactly known and we approximate it using linear regression. 
Assuming the knowledge of $N_{\min}$ and $N_{\max}$, we updated the initializations and Bellman updates for M-Q-M.
%for the iteration process. Depending on how the target reward function is altered, noise might help or hinder the pruning of actions. Depending on how close source behaviors policy are to target behavior, SFQL may perform better or worse respectively. As long as some actions are pruned, Q-M would outperform QL as established theoretically. 
The convergence plots are presented in Fig. \ref{approx},
demonstrating the diminishing boost in performance as the number of actions pruned decreases with increase in the magnitude of noise.
% where the noise levels with respect to the uniform noise added to the rewards was labeled and its corresponding number of actions pruned are shown. 
As expected, we observe that noise has an impact on the efficacy of M-Q-M: the more noise, the smaller the performance gain. 
\textcolor{black}{It is important to note that the maximum magnitude of noise that allows action pruning depends on MDP reward design.}
However, it is promising to observe that M-Q-M can still be effective under such noisy situations since it can greatly expand the applicability of M-Q-M. 
For instance, when the functional relationship is unknown, we can apply regression to fit the source reward functions to the observed target rewards under an assumed functional form based on domain expertise; noise can be incorporated to handle regression error.
}
\begin{figure*}[!htb]
    \centering
    \includegraphics[scale=0.27]{combined_bar_reward_plot2.jpg}
    \vskip -8pt
\caption {Performance under varying noise in auto-generated domains. Left: Actions pruned (\%) vs. noise. Middle: M-Q-M convergence. Right: Q-M convergence. 
% Line color intensity reflects pruning level 
(color: dark = M-Q-M, lighter = Q-M).}
    \label{approx}
\end{figure*}

% \color{red}
\section{Limitations and Future Work}\label{limitations}

\label{Discussion}
% The reduction of the action set can be seen as a form of abstraction, where irrelevant or unhelpful information is effectively filtered out. 
% Just as humans leverage prior experience to disregard futile actions, reinforcement learning agents can similarly benefit from focusing only on meaningful choices. Q-M offers a sound theoretically grounded framework to enable such efficient transfer in RL. However, realizing its full potential in real-world applications entails addressing several key challenges and limitations.
% Built on this solid foundation, Q-M paves the way for new opportunities and challenges in advancing transfer learning within reinforcement learning. 
% \sout{Like humans using prior experience to avoid futile actions, RL agents benefit from focusing on meaningful choices. Q-M provides a theoretically grounded framework for efficient transfer, but realizing its full potential in real-world settings requires addressing key challenges.}

\textbf{Convergence Quality for M-Q-M}:
% \textcolor{blue}{Talk about the convergence quality for M-Q-M. When its initialization is better than the convergence point of M-Q, it will clearly be better. What about the other cases (especially when some of values are better and some are worse)? This can be posed as an interesting question for future work. }
\textcolor{black}{The convergence behavior of M-Q-M depends critically on its initialization relative to the convergence point of Q-M. In the case where the initialization of M-Q-M strictly dominates the fixed point of Q-M (i.e., higher for upper bounds and lower for lower bounds,$\forall s,a$), 
% M-Q-M should perform at least as well as Q-M
M-Q-M would perform at least as well as Q-M. 
% \sout{Since M-Q-M takes the minimum (or maximum) over two estimates, it can never be worse, and when one of the estimates is already tighter (e.g. $a>b$ for lower bounds or $a<b$ for upper bounds, where $b$ is the quantity computed by Q-M and $a$ retains the current estimate), it yields a strictly better value and often better, due to tighter propagated bounds throughout the computation.}
\textcolor{black}{However, the situation is less clear when the quality of the initialization is less known.}
% , and often better, due to tighter propagated bounds throughout the computation.
This raises an open question regarding the quality of convergence for M-Q-M: Under what conditions on the initialization does M-Q-M yield strictly tighter bounds or improved performance compared to Q-M? Formal characterization of such conditions, potentially in terms of partial dominance or monotonicity properties over subsets of the state-action space, could offer deeper theoretical insights and guide more effective initialization strategies. Exploring this direction represents a compelling avenue for future work.
}

\textbf{Safety Considerations:} 
Safety is critical when applying Q-M in real-world settings. 
For instance, in autonomous driving, $Q_{i}^*$ may prioritize avoiding obstacles, ensuring safe navigation. In contrast, learning $Q_{i}^{\mu}$ could result in reckless behavior, like colliding with obstacles, which is unsafe and undesirable. Therefore, selecting a safe, ethical behavior for \textcolor{black}{learning} Q-variants 
% \sout{, such as fast and slow,}
becomes essential when designing 
% \sout{reinforcement learning} 
\textcolor{black}{Q-M} systems for real-world applications.
\textcolor{black}{One possible solution is to leverage safe RL methods, such as shielding~\cite{alshiekh2018safe}, to ensure safety while learning Q-variants.}
% \sout{Learning Q-variants must avoid unsafe objectives, using safety mechanisms like shielding \cite{alshiekh2018safe}.} 
% This is especially important in domains like robotics and autonomous driving. 
% Simulated environments, by contrast, allow more flexibility for experimenting with behavior combinations and pruning.
% This emphasizes the importance of carefully designing source behaviors and target combinations with safety in mind, especially when interacting with the real world. In contrast, if working within a simulation, safety concerns are less critical.

\textbf{Scaling to Real-World Domains:}
While effective in discrete settings, Q-M must overcome challenges to scale to continuous spaces, where Q-value bounds are harder to initialize and refine. This raises a central challenge: how to effectively initialize and refine Q-value bounds when the space is uncountably infinite. One solution lies in leveraging function approximation to estimate these bounds efficiently. Additionally, identifying one-step reachable neighbors becomes more complex in continuous spaces. A potential approach is to approximate these neighbors via uniform sampling within a bounded radius in the state space, or to derive closed-form functions that produce finite, representative next states. 
% Exploring how well Q-M can scale under these conditions, while still improving performance via pruning, remains an open and compelling question.
% \textbf{Approximation error:}
Approximation errors may cause Q-M to prune optimal actions, degrading performance. To ensure reliability, pruning should be used cautiously and combined with techniques that mitigate approximation bias. 
\textcolor{black}{ For example, setting $\Delta$ (refer Sec. \ref{pruningT}) should be done empirically 
due to the fact that the error in the estimated value function is unknown. 
% due to the unavailable $\epsilon$
}
% One solution is to set $\Delta$ threshold (refer Sec. \ref{pruningT}) as a hyperparameter, which should be set empirically due to the unavailable $\epsilon$

\textbf{Generalization and Domain Adaptation:}
Q-M currently assumes a known relationship between source and target rewards, limiting its use when target rewards are unknown.
% , dynamic, or only partially observable. 
Extending Q-M to learn or infer these relationships would improve adaptability. Its pruning strategy could also support dynamics transfer (off-dynamics RL), expanding applications to areas like sim2real transfer.
% Similarly, exploring better initialization techniques without making more assumptions would make Q-M stronger. Additionally, investigating how Q-M integrates with model-based methods or scales to continuous action spaces could further enhance its utility.

Q-M provides a promising foundation for transfer in RL. Realizing its practical impact will require addressing safety, scalability, approximation errors, and generalization, each offering challenges and opportunities for broader real-world adoption.

\color{black}

\section{Conclusions}
% \vskip-5pt
In this paper, we studied reward adaptation, the problem where the learning agent adapted to a target reward function based on the existing source behaviors under the same MDP except for $R$.
% $\setminus R$.
% \sout{We proposed a new approach to reward adaptation, referred as Q-Manipulation (Q-M). 
% The key was to maintain $Q$ variants for each of the source behaviors and apply Q-M iterations to compute bounds of the target $Q$ function and their initializations for action pruning before learning the target behavior. }
\textcolor{black}{We propose 2 methods 1) Q-Manipulation (Q-M) and its extension, 2) Monotonic Q-Manipulation (M-Q-M) as novel, theoretically grounded approaches for reward adaptation in reinforcement learning. By leveraging source Q-function variants, these methods compute tight bounds on the target Q-function to safely prune suboptimal actions.}
We formally proved that our approach converged and retained optimality under correct initializations.
Empirically, we showed that Q-M and M-Q-M were substantially more efficient than the baselines in domains where the source and target behaviors differ, and generalizable under different randomizations. 
We also applied Q-M to noisy combination functions 
to extend its applicability.  
As such, 
% \sout{Q-Manipulation represents a valuable contribution to advancing transfer learning for reinforcement learning.}
\textcolor{black}{our methods offer a robust framework for leveraging prior knowledge in reinforcement learning, advancing the state of transfer and continual learning.}
Our work also opens up many future opportunities, such as addressing continuous state and action spaces and handling different domain dynamics (in addition to reward functions) as in domain adaptation.

%%%%%%%%%%%%%%%%%%%%%%%%%%%%%%%%%%%%%%%%%%%%%%%%%%%%%%%%%%%%
\newpage

\bibliography{main}
\bibliographystyle{tmlr}

%%%%%%%%%%%%%%%%%%%%%%%%%%%%%%%%%%%%%%%%%%%%%%%%%%%%%%%%%%%%%%%%%%%%%%%%%%%%%%%
%%%%%%%%%%%%%%%%%%%%%%%%%%%%%%%%%%%%%%%%%%%%%%%%%%%%%%%%%%%%%%%%%%%%%%%%%%%%%%%
% APPENDIX
%%%%%%%%%%%%%%%%%%%%%%%%%%%%%%%%%%%%%%%%%%%%%%%%%%%%%%%%%%%%%%%%%%%%%%%%%%%%%%%
%%%%%%%%%%%%%%%%%%%%%%%%%%%%%%%%%%%%%%%%%%%%%%%%%%%%%%%%%%%%%%%%%%%%%%%%%%%%%%%
\newpage
\appendix
\onecolumn
%\section{You \emph{can} have an appendix here.}
\section{Appendix}
\subsection{Theoretical Proofs} \label{proofs}

\textbf{Lemma \ref{lemma1}}
\begin{equation}
    \begin{aligned}
          Q^{\mu}_{R}(s,a) &= \min_{\pi} \left[ \mathbb{E} \left[ \sum_{t=0}^{\infty} \gamma^t r_t | s_0, \pi \right] \right] 
    \\
      = &- \max_{\pi} \left[ \mathbb{E} \left[ \sum_{t=0}^{\infty} -\gamma^t r_t | s_0, \pi \right] \right] 
      \\
      &= -Q^*_{-R}(s,a)
    \end{aligned}
\end{equation}

\textbf{Lemma \ref{l4}}
    When $\mathcal{R} = \sum c_iR_i$ where $c_i\geq 0$
    %a  R_{i} + b R_{j}$ $(a > 0, b > 0)$
    , an upper and lower bound of $Q^*_\mathcal{R}$ are given, respectively, by:
\begin{align}
    \begin{aligned}
        Q_0^{UB} & = \sum_{i=1}^{n} c_i Q^{*}_{i}
        \text{\;\;}
        \\
    %\end{aligned} \\
    %\begin{aligned}
        Q_0^{LB}& = \max_i [c_i Q_i^{*} + \sum_j c_j Q_j^{\mu}] \text{\;\;where\;}j\in \{1:n\} \setminus i
    \end{aligned}
\end{align}

\begin{proof}
    %$Q_{i}^{*} $ 
    % $c_i Q_{i}^{*}$
    % is specified as:

From definition, we have:  
\begin{equation}
        \begin{aligned}
            %Q_{i}^{*} 
            c_i Q_{i}^{\pi}
            &= \max_{\pi} \left[ \mathbb{E} \left[c_i r_{{i, 0}} + \gamma c_i r_{{i, 1}} + \ldots + \gamma^{n} c_i  r_{{i, n}}| s_0, \pi \right]\right]
            %\\
            %&=c_i Q_{i}^{*}
        \end{aligned}
\end{equation}
By reorganizing the reward components, we have:
\begin{equation}
        \begin{aligned}
            %Q_{i}^{*} 
            \sum_i c_i Q_{i}^{\pi} = Q^{\pi}_{\sum_i c_i R_i}
            %\\
            %&=c_i Q_{i}^{*}
        \end{aligned}
\end{equation}

Denote the optimal policy under the target reward function $\mathcal{R}$ as $\pi^*$, given $c_i \geq 0$, we can derive that

\begin{equation}
\sum_i c_i Q^{*}_{i} \geq \sum_i c_i Q_{i}^{\pi^*} = Q^{*}_{\mathcal{R}}
\end{equation}
For the lower bound, we have: 
 \begin{equation}
     \begin{split}
         & \max_i (c_i Q_{i}^{*} + \sum_{j \neq i} c_j Q_{j}^{\mu})
        \leq c_k Q_{k}^{*} + \sum_{j \neq k} c_j Q_{j}^{\pi_k^{*}} \\
        &\text{\;\, where $k$ denotes  the best choice of $i$ from the left}\\
         & \leq \max_\pi (c_i Q_{i}^{\pi} + \sum_{j \neq i} c_j Q_{j}^{\pi})
         % \\
         % & \leq \max_{\pi} \left[ \mathbb{E} \left[c_i \sum_i r_{i_{0}} + \gamma c_i \sum_i r_{i_{1}} + \ldots + \gamma^{n} c_i \sum_i r_{i_{n}}| s_0,\pi, r_{i_t} \in R_i \right]\right]
         % \\
         % &\text{\;\, where $i$ denotes  $i^{th}$ source behavior}
         \\
         &= Q_{\mathcal{R}}^{*} 
     \end{split}
 \end{equation}
\end{proof}

Next, we present a few lemmas that are used in the proof of our theorems: 
\begin{lemma}\label{max-max}
$$
\left|\max _a f(a)-\max _a g(a)\right| \leq \max _a|f(a)-g(a)| .
$$
\end{lemma}
\begin{proof}

Assume without loss of generality that $\max _a f(a) \geq \max _a g(a)$, and denote $a^*=\arg \max _a f(a)$. Then,
\[
\left|\max _a f(a)-\max _a g(a)\right|=\max _a f(a)-\max _a g(a)=f\left(a^*\right)-\max _a g(a) \leq f\left(a^*\right)-g\left(a^*\right) \leq \max _a|f(a)-g(a)| .
\]
This concludes the proof.
\end{proof}

\begin{lemma}\label{mim-min}
%Let \( \mathbf{f} = (f_1, f_2, \dots, f_n) \) and \( \mathbf{g} = (g_1, g_2, \dots, g_n) \) be vectors in \( \mathbb{R}^n \). Then
%\[
%\left|\min_i f_i - \min_i g_i\right| \leq \max_i |f_i - g_i|.
%\]
$$
\left|\min _a f(a)-\min _a g(a)\right| \leq \max _a|f(a)-g(a)| .
$$
\end{lemma}

\begin{proof}
Assume without loss of generality that $f(a^*)=\min_a f(a) \geq \min_a g(a)=g(b^*)$.
Then,
$$
\max_a |f(a)-g(a)| \geq\left|f\left(b^*\right)-g\left(b^*\right)\right| \geq f\left(b^*\right)-g\left(b^*\right) \geq f\left(a^*\right)-g\left(b^*\right) 
=%\geq
\left|\min _a f(a)-\min _a g(a)\right|
$$
This concludes the proof.
\end{proof}

%%%%%%%%%%%%%%%%%%%%%%%%%%%%%%%%%%%%%%%%%%%%%%%%%%%%%%%%%%%%%
% \color{red}
\textbf{Theorem \ref{qms}} [Q-M Convergence]
$\mathcal{T}:\mathbb{R}^{|S\times A|} \rightarrow \mathbb{R}^{|S\times A|}$ is a strict contraction
such that the $Q$ function  converges to a unique fixed point for UB and LB, respectively, or more formally:
\[
\begin{aligned}
\|\mathcal{T'} Q_k - \mathcal{T'} Q_{k+1}\|_{\infty} 
&\leq \gamma \|Q_k - Q_{k+1}\|_{\infty}, 
\forall Q_k, Q_{k+1} &\in \mathbb{R}^{|S \times A|}
\end{aligned}
\]
% $\mathcal{T}:\mathbb{R}^{|S\times A|} \rightarrow \mathbb{R}^{|S\times A|}$ is a strict contraction
% such that the $Q$ function converges to a unique fixed point. 
% % \[
% \begin{equation*}
% \begin{aligned}
% \|\mathcal{T'} Q_k - \mathcal{T'} Q_{k+1}\|_{\infty} 
% &\leq \gamma \|Q_k - Q_{k+1}\|_{\infty}, 
% \forall Q_k, Q_{k+1} &\in \mathbb{R}^{|S \times A|}
% \end{aligned}
% % \]
% \end{equation*}
% % \]
\begin{proof}
For any two Q-functions, denoted $Q^{LB}_A$ and $Q^{LB}_B$, the distance between their transformations under the operator is strictly smaller than their original distance, scaled by $\gamma$. We demonstrate this for $\mathcal{T}_{min}$. The proof for $\mathcal{T}_{max}$ follows a symmetric argument.

We want to show:
\[
\| \mathcal{T}_{min} Q^{LB}_A - \mathcal{T}_{min} Q^{LB}_B \|_\infty \le \gamma \| Q^{LB}_A - Q^{LB}_B \|_\infty.
\]

Consider the absolute difference for an arbitrary state-action pair $(s, a)$:
\begin{align*}
    &\left| (\mathcal{T}_{min} Q^{LB}_A)(s, a) - (\mathcal{T}_{min} Q^{LB}_B)(s, a) \right| \\
    &= \left| \min_{s' \in \hat{T}(\cdot|s,a)} \left[ \mathcal{R}(s,a,s') + \gamma \max_{a'} Q^{LB}_A(s', a') \right]
           - \min_{s' \in \hat{T}(\cdot|s,a)} \left[ \mathcal{R}(s,a,s') + \gamma \max_{a'} Q^{LB}_B(s', a') \right] \right| \\
    &\le \max_{s' \in \hat{T}(\cdot|s,a)} \left| \left[ \mathcal{R}(s,a,s') + \gamma \max_{a'} Q^{LB}_A(s', a') \right]
           - \left[ \mathcal{R}(s,a,s') + \gamma \max_{a'} Q^{LB}_B(s', a') \right] \right| \text{\;\; (Lemma \ref{mim-min})} \\
    &= \max_{s' \in \hat{T}(\cdot|s,a)} \left| \gamma \max_{a'} Q^{LB}_A(s', a') - \gamma \max_{a'} Q^{LB}_B(s', a') \right| \\
    &= \gamma \max_{s' \in \hat{T}(\cdot|s,a)} \left| \max_{a'} Q^{LB}_A(s', a') - \max_{a'} Q^{LB}_B(s', a') \right| \\
    &\le \gamma \max_{s' \in \hat{T}(\cdot|s,a)} \max_{a'} \left| Q^{LB}_A(s', a') - Q^{LB}_B(s', a') \right| \text{\;\; (Lemma \ref{max-max})} \\
    &\le \gamma \| Q^{LB}_A - Q^{LB}_B \|_\infty
\end{align*}

This inequality holds for all $(s,a)$, so taking the maximum over all state-action pairs gives:
\[
\| \mathcal{T}_{min} Q^{LB}_A - \mathcal{T}_{min} Q^{LB}_B \|_\infty \le \gamma \| Q^{LB}_A - Q^{LB}_B \|_\infty.
\]

Since $\gamma \in [0, 1)$, the operator $\mathcal{T}_{min}$ is a contraction. By the Banach Fixed-Point Theorem, this implies that $\mathcal{T}_{min}$ has a unique fixed point in the space of bounded Q-functions.
\end{proof}

%%%%%%%%%%%%%%%%%%%%%%%%%%%%%%%%%%%%%%%%%%%%%%%%%%%%%%%%%%%%%%

\textbf{Theorem \ref{ordering}}
% Given $T$(standard bellman operator), $\mathcal{T}_{min}$, and $\mathcal{T}_{max}$ and value iteration sequences be defined by $Q_{k+1} = \mathcal{T} Q_k$, $Q_{k+1}^{LB} = \mathcal{T}_{min} Q_k^{LB}$, and $Q_{k+1}^{UB} =\mathcal{T}_{max} Q_k^{UB}$. If the initial Q-functions are ordered such that $Q_0^{LB} \le Q_0 \le Q_0^{UB}$, then for the unique fixed points $Q^{*}$, $Q^{* \;LB}$, and $Q^{* \;UB}$ of the respective operators, the following inequality holds:
% $$
% Q^{* \;LB} \le Q^{*} \le Q^{* \;UB}
% $$
Given the standard Bellman operator $T$, $\mathcal{T}_{min}$, and $\mathcal{T}_{max}$ (discussed above),
if the initial Q-functions are ordered such that $Q_0^{LB} \le Q_0 \le Q_0^{UB}$, then
$$
Q_k^{LB} \le Q_k \le Q_k^{UB}
$$
holds for all $k\geq0$. Thus, the fixed points $Q^*$, $Q^{LB}$, and $Q^{UB}$ produced by the respective operators,
satisfy the same ordering in the limit as $k \to \infty$.
\begin{proof}
First, we establish two key properties of the operators.
% \textcolor{blue}{Where do you prove Lemma A.3? If it is based on existing results, need to cite.}
\begin{lemma}[Operator Properties]
\
\begin{enumerate}
    \item \textbf{Monotonicity:} For any two Q-functions $Q_A, Q_B$ such that $Q_A \le Q_B$ (pointwise), we have ${T}_{\Omega} Q_A \le T_{{\Omega}} Q_B$ for any of the three operators ${\Omega} \in \{\text{standard}, LB, UB\}$. This is because $\max_{a'} Q_A \le \max_{a'} Q_B$, and the $\min$, $\max$, and expectation operators all preserve this inequality. (for formal proof, refer to proposition 3.6 in \cite{kadurha2025bellman})
    \item \textbf{Ordering:} For any Q-function $Q$, we have $\mathcal{T}_{min} Q \le T Q \le \mathcal{T}_{max} Q$. This follows from the definition of expectation that the minimum of a set of values is less than or equal to their expectation, which is less than or equal to their maximum.

    % \begin{proof}
    % This follows from the property that the minimum of a set of values is less than or equal to their expectation, which is less than or equal to their maximum.
    % By the definition of expectation, for any function $f(s')$ over states $s' \sim T(\cdot|s,a)$:
    % $$
    % \min_{s'} f(s') \le \mathbb{E}_{s'}[f(s')] \le \max_{s'} f(s')
    % $$
    % Let $f(s') = R(s,a,s') + \gamma\max_{a'} Q(s', a')$. 
    % % Adding the common reward and discounted factor yields:
    % \begin{equation*}
    %     R(s,a,s') + \gamma \min_{s'} \max_{a'} Q(s',a') \le R(s,a,s') + \gamma \mathbb{E}_{s'}[\max_{a'} Q(s',a')] \le R(s,a,s') + \gamma \max_{s'} \max_{a'} Q(s',a')
    % \end{equation*}
    % \begin{equation*}
    %     (\mathcal{T}_{\min} Q)(s,a) \le (TQ)(s,a) \le (\mathcal{T}_{\max} Q)(s,a)
    % \end{equation*}
    % % \begin{align*}
    % %    R(s,a) + \gamma \min_{s'} \max_{a'} Q(s',a') &\le R(s,a) + \gamma \mathbb{E}_{s'}[\max_{a'} Q(s',a')] \le R(s,a) + \gamma \max_{s'} \max_{a'} Q(s',a') \\
    % %     (\mathcal{T}_{\min} Q)(s,a) &\le (TQ)(s,a) \le (\mathcal{T}_{\max} Q)(s,a) \quad 
    % % \end{align*}
    % \end{proof}
    
\end{enumerate}
\end{lemma}

To prove $Q^{LB} \leq Q^* \leq Q^{UB}$
We proceed by mathematical induction:

\textbf{Base Case (k=0)}:
$$
Q_0^{LB} \leq Q_0 \leq Q_0^{UB} \text{\;\; (given)}
$$

\textbf{Inductive Hypothesis}:
Assume that for some $k \geq 0$, we have the inductive hypothesis:
$$
Q_k^{LB} \leq Q_k \leq Q_k^{UB}
$$

\textbf{Inductive Step}:
First, we prove the left-hand inequality, $Q_{k+1}^{LB} \le Q_{k+1}$:
\begin{align*}
    Q_{k+1}^{LB} &= \mathcal{T}_{min} Q_k^{LB} & (\text{by definition}) \\
    &\le \mathcal{T}_{min} Q_k & (\text{by monotonicity and induction hypothesis}) \\
    &\le {T} Q_k & (\text{by operator ordering}) \\
    &= Q_{k+1} & (\text{by definition})
\end{align*}
% \textcolor{blue}{the second line above must be based on $Q_k^{LB} \leq Q_k$, right, which is assumed based on induction? If so, you will need to make it much clearer, such as what the inductive hypothesis is and what the initial condition is. }

Next, we prove the right-hand inequality, $Q_{k+1} \le Q_{k+1}^{UB}$:
\begin{align*}
    Q_{k+1} &= {T} Q_k & (\text{by definition}) \\
    &\le \mathcal{T}_{max} Q_k & (\text{by Operator Ordering}) \\
    &\le \mathcal{T}_{max} Q_k^{UB} & (\text{by Monotonicity  and induction hypothesis}) \\
    &= Q_{k+1}^{UB} & (\text{by definition})
\end{align*}

Since both sides of inequalities hold, we have:
$$
Q_{k+1}^{LB} \leq Q_{k+1} \leq Q_{k+1}^{UB}
$$
Thus, by the principle of mathematical induction, the ordering
$$
Q_k^{LB} \leq Q_k \leq Q_k^{UB}
$$
holds for all $k\geq0$.
% Thus, we have shown that $Q_{k+1}^{LB} \le Q_{k+1} \le Q_{k+1}^{UB}$. 
As $k \to \infty$, the inequality also holds for fixed points of individual operators.
\end{proof}
\color{black}
%%%%%%%%%%%%%%%%%%%%%%%%%%%%%%%%%%%%%%%%%%%%%%%%%%%%%%%%%%%%%%%%%%%%%%%%%%%%%%%%%%%%%%%%%%%%%%%%%%%%%

\textbf{Theorem \ref{th1}} [Convergence] 
The iteration process introduced by the Bellman operator in M-Q-M satisfies
% $\mathcal{T}$ 
% converges to a fixed point
% %is a contraction mapping 
% under sup-norm $\|\cdot\|_{\infty}$, 
% %i.e., there exists $\gamma \in[0,1)$ 
% such that for both $Q^{UB}$ and $Q^{LB}$
$$
\|\mathcal{ T} Q_k-\mathcal{T} Q_{k+1}\|_{\infty} \leq 
\gamma
\|Q_k-Q_{k+1}\|_{\infty}, \forall Q_k, Q_{k+1} \in \mathbb{R}^{|S\times A|} .
$$
such that the $Q$ function converges to a fixed point.
\begin{proof}

\textbf{1) Upper Bound}

%\hspace{-0.5cm}
{The operator \(\mathcal{T}_{\min}\) for the upper bound is defined as follows:}
%$$
\begin{equation} \label{operator}
    Q^{UB}_{k+1}(s,a) = (\mathcal{T}_{\min} Q^{UB}_k)(s,a) = \min \left(Q^{UB}_k(s,a), \max_{s' \in {\hat{T}(\cdot|s,a)}} \left[\mathcal{R}(s,a,s') + \gamma \max_{a'} Q_k^{UB}(s',a')\right]\right)
\end{equation}
%$$
where $\hat{T}(\cdot|s,a)$ denotes reachable states from $s,a$. %({\textcolor{red}{I'll refine this notation later}})

We consider the change of difference between $Q$ values between before and after the modified Bellman update (i.e., the difference between $\left|Q^{UB}_{k}(s,a) - Q^{UB}_{k+1}(s,a) \right|$ and $\left|Q^{UB}_{k+1}(s,a) - Q^{UB}_{k+2}(s,a) \right|$):

\textbf{Case 1: }
If the first elements were the smaller values for computing both $Q^{UB}_{k+1}$ and $Q^{UB}_{k+2}$ in Eq. \ref{operator}:
$$
Q^{UB}_{k+1}(s,a) = Q^{UB}_k(s,a)
$$
$$
Q^{UB}_{k+2}(s,a) =  Q^{UB}_{k+1}(s,a)
$$
$$\left| Q^{UB}_{k+1}(s,a) - Q^{UB}_{k+2}(s,a) \right|= |Q^{UB}_{k}(s,a)-Q^{UB}_{k+1}(s,a)|=0$$

\textbf{Case 2:}
If the second element in min was the smaller value for computing $Q^{UB}_{k+1}$ and the first element in min was the smaller value for $Q^{UB}_{k+2}$:
$$
Q^{UB}_{k+1}(s,a) = 
\max_{s' \in {\hat{T}(\cdot|s,a)}} \left[\mathcal{R}(s,a,s') + \gamma \max_{a'} Q_k^{UB}(s',a')\right]
$$
$$
Q^{UB}_{k+2}(s,a) =Q^{UB}_{k+1}(s,a)
$$
$$\left| Q^{UB}_{k+1}(s,a) - Q^{UB}_{k+2}(s,a) \right|=0 
%\leq \gamma \max_{s' \in {\hat{T}(\cdot|s,a)}} \left| Q^{UB}_{k}(s',a^*) - Q^{UB}_{k+1}(s',a^*) \right|
%\leq |Q^{UB}_{k}(s,a)-Q^{UB}_{k+1}(s,a)|
$$

\textbf{Case 3:}
If the first element in min was the smaller value for computing $Q^{UB}_{k+1}$ and the second element in min was the smaller value for $Q^{UB}_{k+2}$:
\begin{equation} \label{case3}
    Q^{UB}_{k+1}(s,a) = Q^{UB}_k(s,a) \leq \max_{s' \in {\hat{T}(\cdot |s,a)}} \left[\mathcal{R}(s,a,s') + \gamma \max_{a'} Q_{k}^{UB}(s',a')\right] \text{(Eq. \ref{operator})}
\end{equation}
$$
Q^{UB}_{k+2}(s,a) =  \max_{s' \in {\hat{T}(\cdot|s,a)}} \left[\mathcal{R}(s,a,s') + \gamma \max_{a'} Q_{k+1}^{UB}(s',a')\right]
$$

%\textbf{Note:} $Q^{UB}_{k+2}(s,a)\leq Q^{UB}_{k+1}(s,a) \leq \max_{s' \in {\hat{T}(\cdot|s,a)}} \left[\mathds{R}(s,a,s') + \gamma \max_{a'} Q_{k}^{UB}(s',a')\right]$. So, an absolute sign doesn't make any difference.

\begin{equation*}
\begin{aligned}
&\left|Q^{UB}_{k+1}(s,a) - Q^{UB}_{k+2}(s,a) \right|
%\left|
\\
&=Q^{UB}_{k}(s,a)-\max_{s' \in {\hat{T}(\cdot|s,a)}} \left[\mathcal{R}(s,a,s') + \gamma \max_{a'} Q_{k+1}^{UB}(s',a')\right] %\geq 0
%\right|
%\\
%&
%\left(\text{given $Q^{UB}_{k+2}(s,a)\leq Q^{UB}_{k+1}(s,a)$ from Eq. \eqref{operator}}\right)
\\
&\leq \max_{s' \in {\hat{T}(\cdot|s,a)}} \left[\mathcal{R}(s,a,s') + \gamma \max_{a'} Q_{k}^{UB}(s',a')\right]-\max_{s' \in {\hat{T}(\cdot|s,a)}} \left[\mathcal{R}(s,a,s') + \gamma \max_{a'} Q_{k+1}^{UB}(s',a')\right] 
\left(\text{Eq. \ref{case3}}\right)
%\\
%&
%\left(\text{using Eq. \eqref{case3}}\right)
\\
&\leq \left|\max_{s' \in {\hat{T}(\cdot|s,a)}} \left[\mathcal{R}(s,a,s') + \gamma \max_{a'} Q_{k}^{UB}(s',a')\right]-\max_{s' \in {\hat{T}(\cdot|s,a)}} \left[\mathcal{R}(s,a,s') + \gamma \max_{a'} Q_{k+1}^{UB}(s',a')\right]\right|
\\
&\leq \gamma \max_{s' \in {\hat{T}(\cdot|s,a)}} \left| \max_{a'} Q^{UB}_{k}(s',a') - \max_{a'} Q^{UB}_{k+1}(s',a') \right|  \text{\; (Lemma \ref{max-max})}
\\
&\leq \gamma \max_{s' \in {\hat{T}(\cdot|s,a)}} \max_{a'} \left|  Q^{UB}_{k}(s',a') - Q^{UB}_{k+1}(s',a') \right|  \text{\; (Lemma \ref{max-max})}
\\
&\leq \gamma \|  Q^{UB}_{k}(s,a) - Q^{UB}_{k+1}(s,a) \|_{\infty} 
\end{aligned}
\end{equation*}

\textbf{Case 4: }
If the second elements in min were the smaller values for both $Q^{UB}_{k+1}$ and $Q^{UB}_{k+2}$:
$$
Q^{UB}_{k+1}(s,a) = \max_{s' \in {\hat{T}(\cdot|s,a)}} \left[\mathcal{R}(s,a,s') + \gamma \max_{a'} Q_{k}^{UB}(s',a')\right]
$$
$$
Q^{UB}_{k+2}(s,a) =  \max_{s' \in {\hat{T}(\cdot|s,a)}} \left[\mathcal{R}(s,a,s') + \gamma \max_{a'} Q_{k+1}^{UB}(s',a')\right]
$$

\begin{equation*}
\begin{aligned}
&\left| Q^{UB}_{k+1}(s,a) - Q^{UB}_{k+2}(s,a) \right| 
\\
&=\left|\max_{s' \in {\hat{T}(\cdot|s,a)}} \left[\mathcal{R}(s,a,s') + \gamma \max_{a'} Q_{k}^{UB}(s',a')\right]-\max_{s' \in {\hat{T}(\cdot|s,a)}} \left[\mathcal{R}(s,a,s') + \gamma \max_{a'} Q_{k+1}^{UB}(s',a')\right]\right|
\\
&\leq \gamma \|  Q^{UB}_{k}(s,a) - Q^{UB}_{k+1}(s,a) \|_{\infty}   \text{\; (similar to Case 3 above)}
\end{aligned}
\end{equation*}

% We use the max norm, which measures the “length” of a vector by the absolute value of its biggest component. Max norm of a vector is given by:
% \begin{equation*}
%     \|Q_{k+1}-Q_{k+2}\|_\infty = max_{s,a} |Q_{k+1}-Q_{k+2}|
% \end{equation*}

Since the above cases hold for any $s, a$, 
we therefore have:

\begin{equation}
    \begin{aligned}
        \| Q^{UB}_{k+1} - Q^{UB}_{k+2}\|_{\infty} 
        % &\leq \max_{s,a}, \gamma \max_{s' \in {\hat{T}(\cdot|s,a)}} \left| Q^{UB}_{k}(s',a^*) - Q^{UB}_{k+1}(s',a^*) \right|
        % \\
        \leq \gamma\|Q_{k}^{UB}-Q_{k+1}^{UB}\|_{\infty} %\text{\;\;(like Q-value iteration)}
    \end{aligned}
\end{equation}

Since the distance decreases by gamma with every iteration,
it will converge to $0$ and hence  $Q^{UB}$ converges to a fixed point.

%%%%%%%%%%%%%%%%%%%%%%%%%%%%%%%%%%%%%%%%%%%%%%%%
%\newline
\textbf{2) Lower Bound}

%\hspace{-0.5cm}
{The operator \(\mathcal{T}_{\max}\) for the lower bound is defined as follows:}
%$$
\begin{equation}\label{operator-lb}
    Q^{LB}_{k+1}(s,a) = (\mathcal{T}_{\max} Q^{LB}_k)(s,a) = \max \left(Q^{LB}_k(s,a), \min_{s' \in {\hat{T}(\cdot|s,a)}} \left[\mathcal{R}(s,a,s') + \gamma \max_{a'} Q_k^{LB}(s',a')\right]\right)
\end{equation}
%$$

$\hat{T}(\cdot|s,a)$ denotes reachable states from $s,a$. %({\textcolor{red}{I'll refine this notation later}})

%\hspace{-0.5cm}
We consider the change of difference between Q values between before and after the modified Bellman update (i.e., the difference between $\left|Q^{LB}_{k}(s,a) - Q^{LB}_{k+1}(s,a) \right|$ and $\left|Q^{LB}_{k+1}(s,a) - Q^{LB}_{k+2}(s,a) \right|$):

%$$\left|Q^{LB}_{k+1}(s,a) - Q^{LB}_{k+2}(s,a) \right|=$$

\textbf{Case 1:}
If the first elements in max were the bigger values for both $Q^{LB}_{k+1}$ and $Q^{LB}_{k+2}$:
$$
Q^{LB}_{k+1}(s,a) = Q^{LB}_k(s,a)
$$
$$
Q^{LB}_{k+2}(s,a) =  Q^{LB}_{k+1}(s,a)
$$
$$\left| Q^{LB}_{k+1}(s,a) - Q^{LB}_{k+2}(s,a) \right|= |Q^{LB}_{k}(s,a)-Q^{LB}_{k+1}(s,a)|=0$$

\textbf{Case 2: }
If the second element in max was the bigger value for $Q^{LB}_{k+1}$ and the first element in max was the bigger value for $Q^{LB}_{k+2}$:
$$
Q^{LB}_{k+1}(s,a) = 
\min_{s' \in {\hat{T}(\cdot|s,a)}} \left[\mathcal{R}(s,a,s') + \gamma \max_{a'} Q_k^{LB}(s',a')\right]
$$
$$
Q^{LB}_{k+2}(s,a) =Q^{LB}_{k+1}(s,a)
$$
$$\left| Q^{LB}_{k+1}(s,a) - Q^{LB}_{k+2}(s,a) \right|=0 %\leq |Q^{LB}_{k}(s,a)-Q^{LB}_{k+1}(s,a)|
$$

\textbf{Case 3: }
If the first element in max was the bigger value for $Q^{LB}_{k+1}$ and the second element in max was the bigger value for $Q^{LB}_{k+2}$:
\begin{equation}\label{case3-lb}
    Q^{LB}_{k+1}(s,a) = Q^{LB}_k(s,a) \geq \min_{s' \in {\hat{T}(\cdot|s,a)}} \left[\mathcal{R}(s,a,s') + \gamma \max_{a'} Q_{k}^{LB}(s',a')\right]
\end{equation}
$$
Q^{LB}_{k+2}(s,a) =  \min_{s' \in {\hat{T}(\cdot|s,a)}} \left[\mathcal{R}(s,a,s') + \gamma \max_{a'} Q_{k+1}^{LB}(s',a')\right]
$$

%\textbf{Note:} $Q^{LB}_{k+2}(s,a)\geq Q^{LB}_{k+1}(s,a) \geq \min_{s' \in {\hat{T}(\cdot|s,a)}} \left[\mathds{R}(s,a,s') + \gamma \max_{a'} Q_{k}^{LB}(s',a')\right]$. %So, an absolute sign doesn't make any difference.

\begin{equation*}
\begin{aligned}
&\left| Q^{LB}_{k+1}(s,a) - Q^{LB}_{k+2}(s,a) \right|
%\left|
\\
&= -\left(
Q^{LB}_{k}(s,a)-\min_{s' \in {\hat{T}(\cdot|s,a)}} \left[\mathcal{R}(s,a,s') + \gamma \max_{a'} Q_{k+1}^{LB}(s',a')\right]
\right)
%\right|
\\
&
\text{\;\;\;\;\;\;\;\;\;\;\;\;\;\;\;\;\;\;\;\;\;\;\;\;\;\;\;\;\;\;\;\;\;\;\;\;\;\;\;\;\;\;\;\;\;\;\;\;\;\;\;\;\;\;\;\;\;\;\;\;\;\;\;\;\;\;\;\;\;\;\;}
\left(\text{since $Q^{LB}_{k+2}(s,a)\geq Q^{LB}_{k+1}(s,a) \text{ based on Eq. } \ref{operator-lb}$%So, an absolute sign doesn't make any difference.
}\right)
\\
&\leq -\left(\min_{s' \in {\hat{T}(\cdot|s,a)}} \left[\mathcal{R}(s,a,s') + \gamma \max_{a'} Q_{k}^{LB}(s',a')\right]-\min_{s' \in {\hat{T}(\cdot|s,a)}} \left[\mathcal{R}(s,a,s') + \gamma \max_{a'} Q_{k+1}^{LB}(s',a')\right]\right) 
\left(\text{Eq. } \ref{case3-lb}\right)
%\\
%&\text{using Eq. \eqref{case3-lb}}
\\
&\leq \left|\min_{s' \in {\hat{T}(\cdot|s,a)}} \left[\mathcal{R}(s,a,s') + \gamma \max_{a'} Q_{k}^{LB}(s',a')\right]-\min_{s' \in {\hat{T}(\cdot|s,a)}} \left[\mathcal{R}(s,a,s') + \gamma \max_{a'} Q_{k+1}^{LB}(s',a')\right]\right|
\\
&\leq \gamma \max_{s' \in {\hat{T}(\cdot|s,a)}} \left| \max_{a'} Q^{LB}_{k}(s',a') - \max_{a'} Q^{LB}_{k+1}(s',a') \right|  \text{\; (Lemma \ref{mim-min})}
\\
&\leq \gamma \max_{s' \in {\hat{T}(\cdot|s,a)}} \max_{a'} \left|  Q^{LB}_{k}(s',a') - Q^{LB}_{k+1}(s',a') \right|  \text{\; (Lemma \ref{max-max})}
\\
&\leq \gamma \|  Q^{LB}_{k}(s,a) - Q^{LB}_{k+1}(s,a) \|_{\infty} 
\end{aligned}
\end{equation*}

\textbf{Case 4: }
If the second elements in max were the bigger values for both $Q_{k+1}$ and $Q_{k+2}$:
$$
Q^{LB}_{k+1}(s,a) = \min_{s' \in {\hat{T}(\cdot|s,a)}} \left[\mathcal{R}(s,a,s') + \gamma \max_{a'} Q_{k}^{LB}(s',a')\right]
$$
$$
Q^{LB}_{k+2}(s,a) =  \min_{s' \in {\hat{T}(\cdot|s,a)}} \left[\mathcal{R}(s,a,s') + \gamma \max_{a'} Q_{k+1}^{LB}(s',a')\right]
$$

\begin{equation*}
\begin{aligned}
&\left| Q^{LB}_{k+1}(s,a) - Q^{LB}_{k+2}(s,a) \right|
\\
&=\left|\min_{s' \in {\hat{T}(\cdot|s,a)}} \left[\mathcal{R}(s,a,s') + \gamma \max_{a'} Q_{k}^{LB}(s',a')\right]-\max_{s' \in {\hat{T}(\cdot|s,a)}} \left[\mathcal{R}(s,a,s') + \gamma \max_{a'} Q_{k+1}^{LB}(s',a')\right]\right|
\\
&\leq \gamma \|  Q^{LB}_{k}(s,a) - Q^{LB}_{k+1}(s,a) \|_{\infty}   \text{\; (similar to Case 3)}
\end{aligned}
\end{equation*}

%The difference between consecutive iterations decreases by gamma for all s,a. As $k\rightarrow \infty$, $Q^{LB}_{K+1}(s,a) - Q^{LB}_{k+2}(s,a) \rightarrow 0$. Therefore, $Q^{LB}(s,a)$ reaches a fixed point [$Q^{LB}_{K+1}(s,a) - Q^{LB}_{k+2}(s,a)\approx 0$].

Since the above cases hold for any $s, a$, 
we therefore have:

\begin{equation}
    \begin{aligned}
        \| Q^{LB}_{k+1} - Q^{LB}_{k+2}\|_{\infty} 
        % &\leq \max_{s,a}, \gamma \max_{s' \in {\hat{T}(\cdot|s,a)}} \left| Q^{UB}_{k}(s',a^*) - Q^{UB}_{k+1}(s',a^*) \right|
        % \\
        \leq \gamma\|Q_{k}^{LB}-Q_{k+1}^{LB}\|_{\infty} %\text{\;\;(like Q-value iteration)}
    \end{aligned}
\end{equation}

Since the distance decreases by gamma with every iteration,
it will converge to $0$ and hence  $Q^{LB}$ converges to a fixed point. 
\end{proof}
% We use the max norm, which measures the “length” of a vector by the absolute value of its biggest component. Max norm of a vector is given by:
% \begin{equation*}
%     \|Q_{k+1}-Q_{k+2}\|_\infty = max_{s,a} |Q_{k+1}-Q_{k+2}|
% \end{equation*}

% Therefore,

% \begin{equation}
%     \begin{aligned}
%         \| Q^{LB}_{k+1} - Q^{LB}_{k+2}\|_{\infty} &\leq \max_{s,a}, \gamma \max_{s' \in {\hat{T}(\cdot|s,a)}} \left| Q^{LB}_{k}(s',a^*) - Q^{LB}_{k+1}(s',a^*) \right|
%         \\
%         &=\gamma\|Q_{k}^{LB}-Q_{k+1}^{LB}\|_{\infty} %\text{\;\;(like Q-value iteration)}
%     \end{aligned}
% \end{equation}

% As shown above, distance decreases by gamma with every iteration. As this distance approaches 0, Q converges to a fixed point.

%Since we have the above for all s,a, we have the conclusion holds. Hence, the difference between consecutive iterations decreases by gamma for all s,a or is zero.

%As $k\rightarrow \infty$, $Q^{UB}_{K+1}(s,a) - Q^{UB}_{k+2}(s,a) \rightarrow 0$. Therefore, $Q^{UB}(s,a)$ reaches a fixed point [$Q^{UB}_{K+1}(s,a) - Q^{UB}_{k+2}(s,a)\approx 0$].
% Therefore, $Q^{LB}(s,a)$ reaches a fixed point

\textcolor{black}{
\textbf{Theorem \ref{unique-fixed-point}}
The Bellman operator in M-Q-M specifies only a non-strict contraction in general:
\begin{equation*}
\left\|\mathcal{T} Q-\mathcal{T} \widehat{Q}\right\|_{\infty} 
     \leq \left\|Q-\widehat{Q}\right\|_{\infty} 
\end{equation*}
}
\begin{proof}
{1) For $\mathcal{T}_{\min}$ computing the upper bound:}
\begin{equation*}
    \begin{split}
        \left|\mathcal{T}_{\min} Q(s, a)-\mathcal{T}_{\min} \widehat{Q}(s, a)\right| =& 
        \bigg|\min \left(Q(s, a), \max_{s' \in {\hat{T}(\cdot|s,a)}}\left[\mathcal{R}(s,a,s')+ \gamma\max_{a'} (Q(s',a'))\right]\right)
        \\
        &-
        \min \left(\widehat{Q}(s, a), \max_{s' \in {\hat{T}(\cdot|s,a)}}\left[\mathcal{R}(s,a,s')+ \gamma\max_{a'} (\widehat{Q}(s',a'))\right]\right)\bigg|
        \\
        &\leq
        \\
        &  {\max}  \bigg( \left|
        %\left|
        Q(s, a) - \widehat{Q}(s,a )
        %\right|
        \right|,
        \\
        &
        \bigg|
        %\bigg|
        \max_{s' \in {\hat{T}(\cdot|s,a)}}\left[\mathcal{R}(s,a,s') + \gamma\max_{a'} (Q(s',a'))\right]
        \\
        & -
        \max_{s' \in {\hat{T}(\cdot|s,a)}}\left[\mathcal{R}(s,a,s') + \gamma\max_{a'} (\widehat{Q}(s',a'))\right]
        %\bigg| 
        \bigg|
        \bigg) 
        \text{\;\;\;\;\;(Lemma \ref{mim-min})}
        \\
         & \leq
         \\
         &  {\max}  \bigg( \left|
        %\left|
        Q(s, a) - \widehat{Q}(s,a )
        %\right|
        \right|,
        \\
        & \gamma \bigg|
        %\bigg|
         \max_{s' \in {\hat{T} (\cdot|s,a)}} \max_{a'} \left[ Q(s',a') -\widehat{Q}(s',a') \right]
        \bigg|
        \bigg) 
        \text{\;\;\;\;\;(Lemma \ref{max-max})}
        \\
        &\leq
        {\max}  \bigg( \left\|
        %\left|
        Q - \widehat{Q}
        %\right|
        \right\|_{\infty},
        \gamma \bigg\|
        %\bigg|
        Q - \widehat{Q}
        %\bigg| 
        \bigg\|_{\infty}
        \bigg) 
        \\
        & = \left\|
        %\left|
        Q - \widehat{Q}
        %\right|
        \right\|_{\infty}
    \end{split}
\end{equation*}

\newpage
{2) For $\mathcal{T}_{\max}$ computing the lower bound:}
\begin{equation*}
    \begin{split}
        \left|\mathcal{T}_{\max} Q(s, a)-\mathcal{T}_{\max} \widehat{Q}(s, a)\right| =& 
        \bigg|\max \left(Q(s, a), \min_{s' \in {\hat{T}(\cdot|s,a)}}\left[\mathcal{R}(s,a,s')+ \gamma\max_{a'} (Q(s',a'))\right]\right)
        \\
        &-
        \max \left(\widehat{Q}(s, a), \min_{s' \in {\hat{T}(\cdot|s,a)}}\left[\mathcal{R}(s,a,s')+ \gamma\max_{a'} (\widehat{Q}(s',a'))\right]\right)\bigg|
        \\
        &\leq
        \\
        &  {\max}  \bigg( \left|
        %\left|
        Q(s, a) - \widehat{Q}(s,a )
        %\right|
        \right|,
        \\
        &
        \bigg|
        %\bigg|
        \min_{s' \in {\hat{T}(\cdot|s,a)}}\left[\mathcal{R}(s,a,s') + \gamma\max_{a'} (Q(s',a'))\right]
        \\
        & -
        \min_{s' \in {\hat{T}(\cdot|s,a)}}\left[\mathcal{R}(s,a,s') + \gamma\max_{a'} (\widehat{Q}(s',a'))\right]
        %\bigg| 
        \bigg|
        \bigg) 
        \text{\;\;\;\;\;(Lemma \ref{max-max})}
        \\
         & \leq
         \\
         &  {\max}  \bigg( \left|
        %\left|
        Q(s, a) - \widehat{Q}(s,a )
        %\right|
        \right|,
        \\
        & \gamma \bigg|
        %\bigg|
         \max_{s' \in {\hat{T} (\cdot|s,a)}} \max_{a'} \left[ Q(s',a') -\widehat{Q}(s',a') \right]
        \bigg|
        \bigg) 
        \text{\;\;\;\;\;(Lemma \ref{mim-min})}
        \\
        &\leq
        {\max}  \bigg( \left\|
        %\left|
        Q - \widehat{Q}
        %\right|
        \right\|_{\infty},
        \gamma \bigg\|
        %\bigg|
        Q - \widehat{Q}
        %\bigg| 
        \bigg\|_{\infty}
        \bigg) 
        \\
        & = \left\|
        %\left|
        Q - \widehat{Q}
        %\right|
        \right\|_{\infty}
    \end{split}
\end{equation*}

Since the above holds for any $s, a$ and for both $\mathcal{T}_{\min}$ and $\mathcal{T}_{\max}$, we have the conclusion holds.
\end{proof}

\textbf{Theorem \ref{th2}} [Optimality]
%Given that $Q_{R_i}$ and $Q_{\left|R_i\right|}$ of each source behavior are accurate, 
For reward adaptation with Q variants, the optimal policies in the target domain remain invariant under Q-M or M-Q-M when the upper and lower bounds are initialized correctly. 
\begin{proof}
%\textbf{Proof:} 
Let %Given the optimal Q function for MDP $<S,A,R,T,\gamma>$ in Eq. \eqref{1},  
\begin{equation*}
    \begin{split}
        %\text{Let, \;\;}
     A_p(s) &=\{\widehat{a}| \text{\;} \exists a \text{\;}{Q}_{}^{LB}(s,a) > {Q}_{}^{UB}(s,\widehat{a}); a\neq\widehat{a}\}
    \\
    \Tilde{A}(s) &= A(s) \setminus A_p(s)
    \end{split}
\end{equation*}
where  $A_p(s)$ represents the set of pruned actions under set $s$ and $\Tilde{A}$ represents the remaining set of actions. To retain all optimal policies,  
it must be satisfied that none of the optimal actions under each state are pruned. 

Assuming that  a pruned action $\widehat{a}$ under $s$ is an optimal action, 
we must have 
\begin{equation*}
    \forall  a \text{\;} Q_{}^{*}(s,a)\leq Q_{}^{*}(s,\widehat{a})
\end{equation*}

Given that Q-M only prunes an action $\widehat{a}$ under $s$ when  $\exists a \text{\;} {Q}_{}^{LB}(s,a) > {Q}_{}^{UB}(s,\widehat{a})$, we can derive that
\begin{equation*}
Q_{}^{LB}(s,a) >
%\mathbb{Q}_{\mathcal{R}_{\mathcal{F}}}^{\mu}(s,a) > \mathbb{Q}_{\mathcal{R}_{\mathcal{F}}}^{*} (s,\widehat{a}) \geq 
Q_{}^{UB}(s,\widehat{a}) \geq Q_{}^{*}(s,\widehat{a}) \geq Q_{}^{*}(s,a),
\end{equation*}
resulting in a contradiction that 
\begin{equation*}
Q_{}^{LB}(s,a) > Q_{}^{*}(s,a)
\end{equation*}
%Therefore, $A^{*} \subseteq \Tilde{A}$. 
As a result, we know that all optimal actions and hence policies are retained. 
\end{proof}

\textbf{Corollary} \ref{nonunique} [Non-uniqueness]
The fixed point of the iteration process in M-Q-M may not be unique.
\begin{proof}
\textcolor{black}{This can be proved using the following example:}

\textcolor{black}{Consider a three state MDP with states s1, s2, s3, where from s1 agent can take an action that transitions uniformly (0.5) to s2 and s3, from s2 agent can take an action that transitions uniformly (0.5) to s1 and s3, and s3 is the terminal state. Reward is 1 for both actions. There is no reward for the terminal state. Assuming a discount factor of 0.5.}

\textcolor{black}{For the upper bound, depending on how V(s3) is initialized, it may result in different fixed points:}
\begin{itemize}
\item \textcolor{black}{When V(s3) is initialized to a big value (say 4), a fixed point may be V(s1) = 3 and V(s2) = 3;}
\item \textcolor{black}{When V(s3) is initialized to a small positive value (say 1), another fixed point could be V(s1) = 3/2 and V(s2) =3/2.}
\end{itemize}
\end{proof}
\subsection{Algorithm}
%Our code is available at: \url{https://anonymous.4open.science/r/Reward_Adaptation_Via_Q-Manipulation/}
\begin{algorithm}[H]
\caption{Reward Adaptation via Q-Manipulation}
\begin{algorithmic}[1]
\STATE Retrieve variants of $Q$, reachable states, and source reward functions from source domains.
\STATE Initialize $Q^{UB}$ and $Q^{LB}$ for the target behavior.
\STATE Tighten the bounds using the iteration process in Q-M or M-Q-M.
\STATE Prune actions.
\STATE Perform learning in the target domain with the remaining actions.
\end{algorithmic}
\end{algorithm}
Github URL: \href{https://anonymous.4open.science/r/Reward-Adaptation-Via-QM-122C/Readme.txt}{https://anonymous.4open.science/r/Reward-Adaptation-Via-QM-122C/Readme.txt}

% \subsection{Additional Information}

\subsection{Domain Information} \label{results}

\subsubsection{MDP Generation}
For autogenerated MDP creation:
the transitions and transition distributions are then randomly generated. 
Initially, the number of reachable states from any $s, a$ is $|A|=9$. However, when an SBF is set for the generated MDP: 
for each $s, a$ pair,
1) we first randomly select a number $k$ from [1, SBF] as the number of reachable states from $s, a$,
2) we retain the state from the transition with the highest probability (which is often the ``intended'' state) while randomly choosing $k-1$ states (without replacement) from its remaining reachable states; these are then considered as the new reachable states from $s, a$, and 3) re-normalize the transition distribution for $s, a$ based on these new reachable states. $3$ states are randomly chosen to be the terminal states. 
Note that a new MDP is generated for each run.
Similarly, for gridworld domains, in each run, the MDP is slightly different with respect to random SBF updated in a manner analogous to the auto-generated MDPs.
For our implementation, we resort to the most general form of reward function $R(s,a,s')$. 

Detailed descriptions of the domains used for our evaluations are given below:

\underline{\textbf{Dollar-Euro:}}
A $45$ states and $4$ actions grid-world domain as illustrated in Fig. \ref{env}.
%\begin{itemize}
    \textbf{Source Domain 1 with $R_1$ (collecting dollars):} The agent obtains a reward of $1.0$ for reaching the location labeled with ``\$", and $0.6$ for reaching the location labeled  with both \$ and \texteuro. %The reward from getting euros is $0$.
    \textbf{Source Domain 2 with $R_2$ (collecting euros):}  The agent obtains a reward of $1.0$ for reaching the location labeled with \texteuro, and $0.6$ for reaching the location labeled  with both \$ and \texteuro. %The reward from getting dollars is $0$.
    \textbf{Target Domain with $\mathcal{R}$:}  $\mathcal{R} = R_{1}+R_{2}$.

\underline{\textbf{Frozen Lake:}}
A standard toy-text environment with $36$ states and $4$ actions. An episode terminates when the agent falls into any hole in the frozen lake ($4$ holes in total) or reaches the goal. 
%\begin{itemize
    \textbf{Source Domain 1 with $R_1$:} The agent is rewarded $+1$ for reaching any hole in a subset of holes (denoted by $H$), $-1$ for reaching any hole in the remaining holes (denoted by $\widehat{H}$) and $0.5$ for reaching the goal.
    \textbf{Source Domain 2 with $R_2$:} The agent is rewarded $+1$ for reaching any hole in $\widehat{H}$, $-1$ for reaching any hole in ${H}$, and $0.5$ for reaching the goal.
    \textbf{Target Domain with $\mathcal{R}$:} Avoid all the holes and reach the goal, or $\mathcal{R} = R_{1}+R_{2}$.
%\end{itemize}
%The living reward is $0$.

\underline{\textbf{Race Track:}}
A $49$ states and $7$ actions grid-world domain. The $7$ actions correspond to different velocities for going forward, turning left, or turning right. An initial location, a goal location, and obstacles make up the race track. An episode ends when the agent reaches the goal position, crashes, or exhausts the total number of steps.
%\begin{itemize}
    \textbf{Source Domain 1 with $R_1$ (avoid obstacles):} The agent obtains a negative reward of $-0.5$ for collision with a living reward of $+0.2$.
    \textbf{Source Domain 2 with $R_2$ (terminate):}  The agent obtains a reward of $+2$ for reaching the goal, $-0.3$ living reward, and $-4$ for staying at the initial location.
    \textbf{Source Domain 3 with $R_3$ (stay put):} The agent obtains a reward of $+3$ for staying at the initial location.
    \textbf{Target Domain with $\mathcal{R}$:} Reach the goal in the least number of steps while avoiding all obstacles, or $\mathcal{R} = R_{1}+R_{2} + R_3$. This is the only domain where there are three source behaviors.
%end{itemize}

\underline{\textbf{Auto-generated Domains:}} 
% We have two different settings for auto-generating domains. These domains all feature two source domains and one target domain. 

% \textbf{Setting 1 (Designed Rewards)}
Generate MDPs with the number of actions=$9$ and the number of states =$60$.
% The transitions and transition distributions are then randomly generated. 
% Initially, the number of reachable states from any $s, a$ is $|A|=9$. However, when an SBF is set for the generated MDP: 
% for each $s, a$ pair,
% 1) we first randomly select a number $k$ from [1, SBF] as the number of reachable states from $s, a$,
% 2) we retain the state from the transition with the highest probability (which is often the ``intended'' state) while randomly choosing $k-1$ states (without replacement) from its remaining reachable states; these are then considered as the new reachable states from $s, a$, and 3) re-normalize the transition distribution for $s, a$ based on these new reachable states. $3$ states are randomly chosen to be the terminal states. 
Rewards for the source domains (i.e., $(R_1, R_2)$) for two of those states are set to ($+1,-1$) and ($-1,+1$), respectively; rewards for the third terminal state are set to $(+0.6, +0.6)$.

\subsection{Learning with Q-variant in Practice}
\textcolor{black}{It is important to note that Q-variants may be difficult to learn with the same samples as experienced during a typical Q-learning process for $Q^*$. 
Some adaptation to Q learning must be made in order to learn $Q^*$ and $Q^{\mu}$ (or other Q-variant)
via the same set of samples. Note that theoretically, Q learning is guaranteed to converge regardless of the behavior policy, although that is inefficient and can result in inaccuracy in practice due to that the behavior policy may result in visiting a different distribution of the states from that of the optimal policy (distributional shift). To ensure that $Q^*$ and $Q^{\mu}$ (or other Q-variant) can both receive informative samples, one possible way is to alternate between training $Q^*$ and $Q^{\mu}$ (or other Q-variant) and use importance sampling while using samples from $Q^{\mu}$ (or other Q-variant) to training $Q^*$ (or vice versa), so that we can leverage samples from both $Q^*$ and $Q^{\mu}$ (or other Q-variant) to train both $Q^*$ and $Q^{\mu}$ (or other Q-variant). Fig. \ref{importance_sampling} shows that approximately the same number of samples are used in the training of individual behaviors for the autogenerated MDP. Moreover, for learning in practice, we rely on memorizing the reward functions during the learning of source behaviors. As a result, Q-M computes bounds using both the learned value functions and the memorized rewards, which has been shown to work well in practice (Fig. \ref{importance_sampling} and \ref{DEReal}).
For the racetrack domain, we observe that optimal actions are pruned, leading to suboptimal convergence as SBF increases. We believe this is due to the approximate estimate of Q in the source domain, which results in incorrect initialization and subsequently pruning of optimal actions due to inaccurate bounds. In such cases, Q-M serves as a better alternative since it does not rely on Q-value initialization.
%Highway City domain.
}
\begin{figure}[!htb]
    \centering
    \includegraphics[scale=0.31]{Aggregated_Realistic_Comparison.jpeg}
    \vskip -8pt
    \caption{Convergence plot for autogenerated domain with linear reward combination: Behavior 1 (left), Behavior 2 (center) and Target (right) where M-Q-M\_p performs action pruning using a estimated  lite model, reward model, and Q-variants. M-Q-M\_p indicates M-Q-M in practice.}\label{importance_sampling}
    % \label{HCDomain}
    % \label{hc}
\end{figure}
\begin{figure}[!htb]
    \centering
    \includegraphics[scale=0.31]{CI_exp1_5+real.jpeg}
    \vskip -8pt
    \caption{Convergence plot for: (top) Dollar Euro, (middle) Racetrack and (bottom) Frozen Lake
    using a estimated  lite model, reward model, and Q-variants. M-Q-M\_p indicates M-Q-M in practice.}\label{DEReal}
    % \label{HCDomain}
    % \label{hc}
\end{figure}

% \begin{figure}[!ht]
%     \begin{subfigure}
%         \centering
%         \includegraphics[scale=0.4]{R1_Qstar-Qmu.png}
%         %\caption{Source Behavior 1}
%         %\label{fig:sub1}
%     \end{subfigure}
%     \hfill
%     \begin{subfigure}
%         \centering
%         \includegraphics[scale=0.4]{R2_Qstar-Qmu.png}
%         %\caption{Source Behavior 2} 
%         %\label{fig:sub2}
%     \end{subfigure}
%     \caption{Convergence plot for individual behaviors comparing epsilon greedy and modified epsilon greedy exploration strategy}
%     \label{importance_sampling}
% \end{figure}

\subsubsection{Hyperparameters} \label{hparams}
\textcolor{black}{
All hyperparamters are set to be same for the different methods in the same evaluation domain.
% For continuous domains, the input layer of DQN is followed by $3$ fully connected layers each consisting of $64$ neurons with relu activation. We used a buffer size of $100000$, a batch size of $64$, $\tau=0.001$ for soft update of the target network parameters and a learning rate of $0.0005$.
The exploration rate starts from $1.0$ and is gradually decayed. %in both discrete and continuous domains. 
$\gamma$ is chosen between $[0.9, 0.99]$ across different domains.
}

%These are attached in the zip file as there is a heatmap generated for all 30 runs and respective SBFs. 

%In this section, we present additional results with synthetic domains where MDPs are auto-generated. In Tab. \ref{atable1}, we present results for domains generated with larger state and action space sizes with linear $\mathcal{R}$. In Tab. \ref{atable2}, results for domains with non-linear $\mathcal{R}$ are presented where we allow the source rewards to be both positive and negative. Note that this conflicts with the theoretical results and hence may lead to the loss of optimality. In practice, for the domains we tested, Q-M still produced the optimal solutions. Convergence comparisons between Q-M and the baselines under these domains are presented in Figs. \ref{large-domain} and \ref{non-linear}, respectively. Similar observations were made. These results demonstrated that Q-M can scale to larger domains and that it is somewhat robust to the assumption concerning positive source reward functions in the non-linear target reward function setting. We will study these further in future work. 

\subsubsection{Running Time Comparison}
We measured the running times taken to run each evaluation for each method for a fixed number of training steps on an XPS 9500 laptop. The aim here is to show that Q-M iteration adds, in most cases, a reasonable amount of extra computation to the entire learning process. 

\begin{table}[!htb]
\footnotesize
\centering
\begin{tabular}{|l|c|c|c|}
\hline
\textcolor{black}{Domain}                               & \textcolor{black}{SBF\_min} & \textcolor{black}{SBF\_mid} & \textcolor{black}{SBF\_max} \\ \hline
\textcolor{black}{Dollar Euro}                           & \textcolor{black}{0.05}     & \textcolor{black}{0.04}     & \textcolor{black}{0.03}     \\ \hline
\textcolor{black}{Race Track}                            & \textcolor{black}{0.13}     & \textcolor{black}{0.12}     & \textcolor{black}{0.15}     \\ \hline
\textcolor{black}{Frozen Lake}                           & \textcolor{black}{0.04}     & \textcolor{black}{0.04}     & \textcolor{black}{0.04}     \\ \hline
\textcolor{black}{Autogenerated}                         & \textcolor{black}{0.11}     & \textcolor{black}{0.11}     & \textcolor{black}{0.12}     \\ \hline
% \textcolor{black}{Autogenerated (randomized MDP and R)}  & \textcolor{black}{2.54}     & \textcolor{black}{3.17}     & \textcolor{black}{2.92}     \\ \hline
\textcolor{black}{Non-linear Target Reward}               & \textcolor{black}{0.23}     & \textcolor{black}{1.2}     & \textcolor{black}{0.42}     \\ \hline
\textcolor{black}{Noisy Reward Combination}              & \textcolor{black}{0.12}     & \textcolor{black}{0.12}     & \textcolor{black}{0.13}   
\\ \hline
\end{tabular}
\caption{Running time for Q-M iteration process 
(in seconds) by Domain and SBF
}
\label{running1}
\end{table}

% \begin{table}[!ht]
% \centering
% \begin{tabular}{|l|l|l|l|l}
% \cline{1-4}
% Domain       & Init   & Iteration & Total  &  \\ \cline{1-4}
% Lunar Lnader & 205.86 & 257.75    & 463.61 &  \\ \cline{1-4}
% Cartpole     & 14.46  & 166.9     & 181.36 &  \\ \cline{1-4}
% Mountain Car & 1.25   & 160.05    & 161.3  &  \\ \cline{1-4}
% \end{tabular}
% \caption{Running time (in mins) for Q-M bounds computation in continuous domain}
% \end{table}

%%%%%%%%%%%%%%%%%%%%%%%%%%%%%%%%%%%%%%%%%%%%%%%%%%%%%%%%%%%%

\newpage

\end{document}